\documentclass[twoside]{article}

\PassOptionsToPackage{compatV3}{fancyhdr}
\usepackage[preprint]{aistats2026}

\usepackage[utf8]{inputenc}
\usepackage[T1]{fontenc}
\usepackage{hyperref}
\usepackage{url}
\usepackage{booktabs}
\usepackage{amsfonts}
\usepackage{nicefrac}
\usepackage{microtype}
\usepackage{xcolor}

\usepackage{amsmath}
\usepackage{amsthm}
\usepackage{enumerate}
\usepackage{graphicx}
\usepackage{caption}
\usepackage{subcaption}
\usepackage{dsfont}
\usepackage{amssymb}
\usepackage{cleveref}

\usepackage{mathtools}
\usepackage{tikz}

\usepackage[most]{tcolorbox} 

\usepackage{mathrsfs}

\newcommand{\End}{\mathrm{End}}
\newcommand{\Hom}{\mathrm{Hom}}
\newcommand{\erw}{\mathbb{E}}
\newcommand{\id}{\mathrm{id}}
\newcommand{\one }{\mathbf{1}}
\newcommand{\dd}{\mathrm{d}}

\newcommand{\otensor}{\otimes}

\newcommand{\sse}{\subseteq}

\newcommand{\calJ}{\mathcal{J}}
\newcommand{\calS}{\mathcal{S}}

\newcommand{\calX}{\mathcal{X}}
\newcommand{\calH}{\mathcal{H}}
\newcommand{\calY}{\mathcal{Y}}

\newcommand{\calL}{\mathcal{L}}

\newcommand{\calE}{\mathcal{E}}
\newcommand{\sprod}[1]{\langle #1 \rangle}

\newcommand{\C}{\mathbb{C}}
\newcommand{\Hh}{\mathbb{H}}
\newcommand{\F}{\mathbb{F}}

\newcommand{\R}{\mathbb{R}}
\newcommand{\N}{\mathbb{N}}

\newcommand{\SO}{\mathrm{SO}}
\renewcommand{\O}{\mathrm{O}}
\newcommand{\U}{\mathrm{U}}
\newcommand{\Sp}{\mathrm{Sp}}

\DeclareMathOperator{\Span}{span}

\DeclareMathOperator{\range}{range}

\DeclareMathOperator{\Herm}{Herm}
\DeclareMathOperator{\Sym}{Sym}

\newtheorem{example}{Example}
\newtheorem{assumption}{Assumption}

\newtheorem{remark}{Remark}
\newtheorem{lemma}{Lemma}
\newtheorem{theorem}{Theorem}
\newtheorem{infthm}[theorem]{Informal Theorem}
\newtheorem{cor}{Corollary}
\newtheorem{prop}{Proposition}

\newtheorem{definition}{Definition}

\let\svthefootnote\thefootnote
\newcommand\freefootnote[1]{%
  \let\thefootnote\relax%
  \footnotetext{\hspace{-1.5em}#1}%
  \let\thefootnote\svthefootnote%
}

\usepackage{tikz-cd}

\newif\ifJunk
\Junktrue
\usepackage{cancel}

 \usepackage[
    backend=biber,
    style=authoryear,
    maxcitenames=1,
    uniquelist=false
  ]{biblatex}
\newtcolorbox{researchquestion}{
    enhanced,
    boxrule=0pt,
    frame hidden,
    borderline west={2pt}{0pt}{black!70},
    colback=gray!5,
    sharp corners,
    fontupper=\itshape,
    before skip=1em,
    after skip=1em
}

\begin{document}

\runningtitle{Conservation Laws from Data Symmetry}

\twocolumn[

\aistatstitle{Conservation Laws from Data Symmetry in Neural Networks}

\aistatsauthor{ Jakob Galley * \And Vahid Shahverdi * \And  Axel Flinth * }

\aistatsaddress{Ume{\aa} University, Ume{\aa}, Sweden} ]

\begin{abstract}
 We explore whether intrinsic symmetries of the training data lead to conserved quantities during gradient-flow training of neural networks. Under the assumption that the loss function is analytic and non-polynomial, we prove that data symmetries generically do not induce any additional integrals of motion. For mean squared error (MSE) loss, on the other hand, there are situations in which data augmentation yields extra conserved quantities. We build a framework, utilizing \emph{tensorizable networks} to describe this phenomenon. Tensorizable networks are a family of architectures whose dependence on parameters and inputs can be separated using an intermediate representation. They include linear and polynomial networks, as well as Lightning Attention.
\end{abstract}

\section{INTRODUCTION}
When studying gradient-flow training of neural networks -- or more generally any dynamics -- it is useful to find quantities that remain constant along the trajectories. Such conservation laws, or \emph{integrals of motion}, constrain the optimization path and help explain initialization dependence, convergence, and implicit bias \parencite{NIPS2000_13168e6a, arora_implicit_2019, neyshabur_search_2015, soudry_implicit_2024, gunasekar_characterizing_2018,zhang_understanding_2021}. As a concrete example, much of the literature on linear and homogeneous models relies on quantities that are preserved during training  \parencite{saxe2013exact, du2018algorithmic, ji2018gradient,arora2019convergence}.

\begin{figure}[t]
    \centering
    \includegraphics[width=\columnwidth]{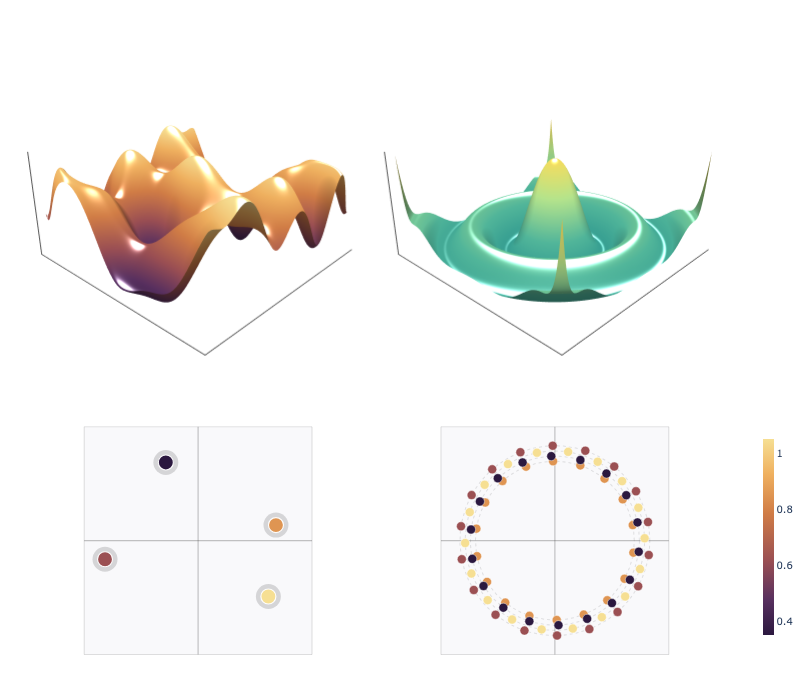}
    \caption{
    A display of how data symmetry can give rise to conservation laws. The top row shows the ordinary and group-augmented loss landscapes for a two-parameter neural network, with the corresponding training data shown below.
    }
    \label{fig:symmetric-loss-landscape}
\end{figure}

Historically, most analyses of this type study integrals of motion that are independent of the training data. This is in particular true for the seminal papers  \parencite{marcotte2023abide, marcotte2024keep}. In many applications, however, the training data is not arbitrary. It often carries inductive bias such as 
symmetry under permutations, rotations, reflections, or other group actions. Moreover, if the learning target is known to have a symmetry a priori -- meaning that it is equivariant with respect to a group action -- then the symmetry can be introduced into training through data augmentation  \parencite{chen2020group, lyle2005benefits, mei2021learning}. \freefootnote{\quad *Equal contribution.}

This motivates the main question of our work:
\begin{researchquestion}
Can symmetries in training data lead to new integrals of motion in gradient-flow training?
\end{researchquestion}

This question is related to the classical relation between symmetries and conservation laws. Noether's theorem \parencite{noether1918invariante} states that, under suitable assumptions, a continuous symmetry of a dynamical system gives rise to an integral of motion. Several works in this direction study symmetries of the parametrization, or of the learning dynamics, and derive restrictions on gradient paths \parencite{gluch2021noether,tanaka2021noether}. Related analyses connect continuous parameter symmetries to broken conservation laws, flat minima, and coordinates along low-loss valleys \parencite{kunin2020neural,zhao2022symmetries}.

However, in these prior works, the symmetry exists in the \emph{parameter space} $\Theta$ of the model -- transforming the parameters according to some group action does not change the function they parametrize. The symmetries we consider here are of a different nature. They are present in the \emph{data} -- transforming it with respect to some group action leaves it invariant.

We note that we are still inspired by Noether. In particular, our analysis for finding new integrals of motion boils down to showing that a symmetry in the data \emph{introduces} a new symmetry on the parameter space. This emergent parameter symmetry is far from obvious, as it is connected with both the geometry of the symmetry group and the network architecture itself.

\subsection{Summary of Results}
Our contribution is theoretical, yielding two main results that investigate when data symmetries do -- and do not -- generate new conservation laws.

The first result concerns cases when data symmetries do not induce any new integrals of motion. We assume that the parameters are initialized outside of exceptional regions where the model already exhibits equivariance-type degeneracies (discussed further in Section \ref{subsec:no_new_conservation_laws}). Since these degenerate loci have measure zero in the parameter space of non-equivariant networks, a standard random initialization will almost surely avoid them.

\begin{infthm}[Section \ref{subsec:no_new_conservation_laws}]
    Symmetric data does not yield any new integrals of motion for any analytic non-polynomial margin loss and finite symmetry group for almost all initialization.
\end{infthm}
Polynomial margin losses, such as the MSE loss, however, act differently. This case is analyzed here only for \emph{tensorizable networks}, i.e., those neural networks whose dependency on the input and on the parameters can be separated by a transformation (see Section \ref{subsec:tensor}). In such networks, a symmetry in the training set gives rise to a new symmetry in the parameter space. Most importantly, the newly induced symmetry in the parameter space may well be continuous, despite the data symmetry being discrete.

\begin{infthm}[Section \ref{subsec:more_integrals}]
\label{infthm:MSE_loss}
    For tensorizable networks trained with MSE loss, data symmetries can create new integrals of motion. 
\end{infthm}

\section{RELATED WORK}

Our framework is related to network identifiability, geometric deep learning, and data augmentation, which we briefly explain here.

\paragraph{Model Symmetry and Identifiability.}
A line of work close to ours studies parameter space symmetries -- weight transformations that leave the realized function invariant. In slightly more formal language, these operations preserve the \emph{fibers} of the neural network parametrization map. \emph{Identifiability} asks whether these fibers are exhausted by the natural symmetries of the model, such as permutations and rescalings \parencite{phuong2020functional,vlavcic2022neural}, or if it also allows for more exotic operations \parencite{grigsby2023hidden, grillo2026most}. Moreover, for certain architectures, such as polynomial and group convolutional networks, identifiability has recently been characterized using algebraic geometry \parencite{shahverdi2026learning,hendi2026geometry}. Our work can be phrased as the data symmetry enlarging the fibers of the training loss in a way that gives a parameter symmetry.

\paragraph{Equivariant Networks.} 
Our work also connects to \emph{geometric deep learning} \parencite{bronstein2021geometric}, which focuses on building equivariance directly into the neural network. Through this approach, one can design architectures that are equivariant to a given symmetry group. A standard recipe for this is to make every individual layer equivariant, leading to architectures like group convolutional networks \parencite{cohen2016group}. A recent work proves that, under certain identifiability assumptions, this layer-wise construction is essentially the only way to build equivariant models \parencite{shahverdi2026identifiable}.

While geometric deep learning requires a priori knowledge of a symmetry group to hardcode it into the architecture, our framework applies to general networks where integrals of motion emerge from symmetries in the training data -- even if that data symmetry is entirely unknown.

\paragraph{Data Augmentation and Network Training.}
The impact of data symmetries has been studied before in the context of data augmentation \parencite{chen2020group, lyle2005benefits, mei2021learning,chen2023implicit,duan2025understanding}. Particularly close to our setting is the comparison between equivariant training and training with augmented data. Under suitable assumptions on the architecture, the space of parameters corresponding to `canonical equivariant networks'  becomes invariant under augmented gradient flow, and regularization can even make the set attractive \parencite{nordenfors2025data,nordenfors2025optimization}. This can be viewed as a conservation-type statement on a special subset of parameters. Our goal is instead to characterize when data symmetry creates new integrals of motion on generic regions of the parameter space.

\section{BACKGROUND}
We introduce the notation and the main concepts used throughout the paper. 

Let $\mathcal X$ be the input space, $\mathcal Y$ the output space, and $\Theta$ the parameter space. We assume that these are finite-dimensional Euclidean spaces. A neural
network parametrization is a {smooth} map
\begin{equation}
\begin{aligned}
    f:\Theta \to
    \mathcal F(\mathcal X,\mathcal Y),\quad 
    \theta \mapsto f_\theta,
\end{aligned}
\end{equation}
where $\mathcal F(\mathcal X,\mathcal Y)$ denotes the space of maps from $\calX$ to $\calY$. Given a {compactly supported} distribution $\pi$ on $\calX \times \calY$ and a smooth loss $\ell:\mathcal Y\times\mathcal Y\to\mathbb R$, we define the empirical risk 
\begin{equation}
    \mathcal E_\pi(\theta) = \mathbb{E}_{(x,y)\sim \pi}( \ell(f_\theta(x),y)).
\end{equation}
For a single data point, we write $\mathcal E_{(x,y)}$ instead of $\mathcal E_{\delta_{(x,y)}}$, where $\delta_{(x,y)}$ is the Dirac distribution concentrated at $(x,y)\in\mathcal X\times\mathcal Y$.

We will consider training under gradient flow, defined by the following ordinary differential equation (ODE)
\begin{equation}
\label{eq:ode}
    \dot\theta = -\nabla \mathcal E_\pi(\theta).
\end{equation}

\subsection{Integrals of motion}
We are interested in functions of the parameters that remain constant along gradient flow. Since our goal is to understand conservation laws forced by the model class rather than by one particular dataset, we follow \parencite{marcotte2023abide} and use the following dataset-independent notion.

\begin{definition}
\label{def:integral_of_motion}
Let $\Omega\subseteq\Theta$ be an open set in Euclidean topology. A smooth function $I:\Omega\to\mathbb R$ is an \emph{integral of motion} if, for every {compactly supported} distribution $\pi$ and every solution of \eqref{eq:ode}, we have
\begin{equation} \label{eq:preserved}
    I(\theta(t))=I(\theta(0))
\end{equation}
for all $t$ for which $\theta(t)\in\Omega$.
\end{definition}

By the chain rule, $I$ is an integral of motion if and only if its gradient is everywhere orthogonal to the training dynamics
\begin{equation}
\label{eq:integral_grad}
    \left\langle \nabla I(\theta), \nabla \mathcal E_\pi(\theta) \right\rangle = 0
\end{equation}
for every distribution $\pi$. Since the risk $\mathcal E_\pi$ is an expectation over individual data points, linearity implies that \eqref{eq:integral_grad} holds if and only if $\langle \nabla I(\theta), \nabla \mathcal E_{(x,y)}(\theta) \rangle = 0$ for all individual data points $(x,y) \in \calX \times \calY$.

\subsection{Data symmetries}
The main focus of this paper is the case where the training data is not arbitrary, but constrained by a symmetry. Let $G$ be a compact group acting linearly on $\calX$ and $\calY$ through orthogonal representations $\rho_{\calX}$ and $\rho_{\calY}$ . A distribution $\pi$ on $\calX\times\calY$ is called symmetric under a group $G$ if applying the group action does not change its law, that is,
\begin{equation}
    (x,y)\sim\pi \quad\Longrightarrow\quad \left( \rho_{\calX}(g)x, \rho_{\calY}(g)y \right)\sim\pi
\end{equation}
for every $g \in G$. 

We say that $I$ is an \emph{integral of motion for symmetric data under $G$} if we have \eqref{eq:preserved} for every \emph{symmetric} distribution $\pi$ and solution \eqref{eq:ode}.

Moreover, this symmetry can be incorporated into the loss by averaging over the group. Let $\mu$ be the normalized Haar measure on $G$ \footnote{ 
A regular Borel measure on $G$ is a \emph{Haar measure} if it is invariant under left and right multiplication by elements of $G$, and \emph{normalized} if it is scaled so that $\mu(G)=1$.}. For any distribution $\pi$, we define the group-augmented risk
\begin{equation}
\label{eq:augmented_risk}
    \calE^G_\pi(\theta) = \erw_{\substack{(x,y)\sim\pi\\ g\sim\mu}} \big( \ell\left(f_\theta(\rho_{\calX}(g)x),
     \rho_{\calY}(g)y \right) \big).
\end{equation}
If $\pi$ is already symmetric under $G$, then the group averaging does not change the risk i.e., $\calE^G_\pi=\calE_\pi$. Conversely, for an arbitrary distribution $\pi$, the augmented risk is the ordinary risk associated with the symmetrized distribution obtained by averaging $\pi$ over the group. This gives the following equivalent formulation.

\begin{prop} \label{prop:augflow}
    A function $I$ is an integral of motion for symmetric data under $G$ if and only if for every distribution $\pi$ on $\calX\times\calY$, it is preserved by the augmented gradient flow
    \begin{equation}
        \dot{\theta}= - \nabla \calE_\pi^G(\theta).
    \end{equation}
\end{prop}
The detailed proof is given in Appendix \ref{app:otherproofs}.

By Proposition \ref{prop:augflow}, a \emph{new} integral of motion exists only if a function is conserved under the augmented risk $\mathcal E^G$ but not under the standard risk $\mathcal E$. Following the geometric condition in \eqref{eq:integral_grad}, this requires the augmented gradient flow to span a strictly smaller subspace of the parameter tangent space i.e.,
\begin{equation} \label{eq:span}
    \operatorname{span}_{(x,y)}
    \nabla\mathcal E^{G}_{(x,y)}(\theta)
    \subsetneq
    \operatorname{span}_{(x,y)}
    \nabla\mathcal E_{(x,y)}(\theta),
\end{equation}
where the span is taken over $(x,y)\in \calX \times \calY$. 
Only when \eqref{eq:span} holds may  there be functions whose gradients are orthogonal to all augmented, but not all original, directions.

\subsection{Tensorizable networks}
\label{subsec:tensor}
Here, we introduce a class of models for which the symmetry induced by augmentation can be studied through a lifted feature space. 

\begin{definition}
A model $f_\theta:\mathcal X\to\mathcal Y$ is called
\emph{tensorizable} if there exist a vector space $\mathcal H$, a parameter-independent map $T:\mathcal X\to\mathcal H$, and a {smooth} map $M: \Theta \to \mathrm{Hom}(\mathcal H,\mathcal Y)$ \footnote{Here, $\mathrm{Hom}(\mathcal H,\mathcal Y)$ denotes the vector space of linear maps from $\mathcal H$ to $\mathcal Y$.} such that
\begin{equation}
\label{eq:tensor_form}
    f_\theta(x) = M(\theta)T(x)
\end{equation}
for all $\theta\in\Theta$ and $x\in\mathcal X$.
\end{definition}

Thus the input is first lifted to a feature space $\calH$, after which the parameters act linearly on the lifted feature. The map $M(\theta)$ may still depend nonlinearly on the original parameters, as can $T(x)$ on the data.

Linear and polynomial models, which are among the core models in neuroalgebraic geometry \parencite{marchetti2025algebra}, are tensorizable.  More generally, if the output of a model depends on the input as a polynomial of degree at most $k$, one may take $T(x)$ to be the vector of all input monomials up to degree $k$,
\begin{equation}
    T(x) = (1,x,x^{\otimes 2},\ldots,x^{\otimes k}),
\end{equation}
while $M(\theta)$ collects the corresponding parameter-dependent coefficients. 

In the example below, we construct the tensorized form for Lightning Attention. 

\begin{example}
\label{ex:tensorized_lightning}
Let $\mathcal X=\mathbb R^{n\times d}$ and $\mathcal Y=\mathbb R^{n\times m}$ denote the input and output spaces. Here $n$ is the number of tokens, and $d,m$ are the input and output feature dimensions, respectively. Consider single-head, single-layer Lightning Attention without normalization,
\begin{equation}
    f_{\theta}(X)=XQK^\top X^\top X V,
\end{equation}
where $X\in\mathcal X$, and $\theta = (Q,K,V)$ with $Q,K \in \mathbb R^{d\times r}$ and $V\in \mathbb{R}^{d\times m}$.

We show that this network admits a tensorized form. To do this, let $x_i^\top $ be the rows of $X$ and $Z_X = X^\top X$. The lifted feature keeps track of each token together with the quadratic
feature covariance
\begin{equation}
     T(X) := \bigl( x_i\otimes Z_X\bigr)_{i=1}^n,
\end{equation}
which is an element of the  space 
\begin{equation}
    \mathcal H := \bigl(\mathbb R^d \otimes\operatorname{Sym}^2(\mathbb R^d)\bigr)^n.
\end{equation}
Here, $\operatorname{Sym}^2(\mathbb R^d)$ denotes the space of symmetric second-order tensors over $\mathbb R^d$, identified with the space of symmetric $d\times d$ matrices. 

Now, we construct the linear map $M(\theta)$. First, define $C_\theta$ on $\R^d\otimes \operatorname{Sym}^2(\R^d)$ through 
\begin{equation}
\begin{split}
    C_\theta &: \mathbb R^d\otimes \operatorname{Sym}^2(\mathbb R^d) \to \mathbb R^m, \\
    C_\theta(v\otimes S) &:= V^\top S KQ^\top v .
\end{split}
\end{equation}

Now define $M(\theta)$ through letting $C_\theta$ act component-wise. That is, the $i$-th component of $M(\theta)T(X)$ is
\begin{equation}
    V^\top Z_XKQ^\top x_i,
\end{equation}
which is the transpose of the $i$-th row of
$XQK^\top X^\top XV$. Hence, after identifying
$\mathbb R^n\otimes\mathbb R^m$ with
$\mathbb R^{n\times m}$, we have the tensorized form \eqref{eq:tensor_form}.
\end{example}

\section{RESULTS}
In this section, we provide our main results and state under what conditions data augmentation can or cannot create new integrals of motion.
\subsection{The absence of new conservation laws for analytic margin losses}
\label{subsec:no_new_conservation_laws}

We aim to show that for finite groups and analytic, non-polynomial margin losses, a reduction in the span of gradient directions (as in \eqref{eq:span}) generically \emph{does not} occur. Consequently, data symmetries \emph{do not} induce any new integral of motion, provided we avoid exceptional regions of the parameter space where the network is already equivariant.

Throughout this subsection, we focus on margin-type losses of the form
\begin{equation}
\label{eq:margin_loss_negative}
    \ell(\hat y,y) = \mathcal L(\langle \hat y,y\rangle),
\end{equation}
where $\mathcal L:\mathbb R\to\mathbb R$ is an analytic, non-polynomial function (such as the logistic or exponential loss).

Before stating the main theorem, we must exclude degenerate cases.  As discussed in the related work,  \parencite{nordenfors2025optimization} shows that the subspace of equivariant parameters is invariant under training on symmetric data. Specifically, if parameters are initialized such that $f_\theta(\rho_{\calX}(g)x) = \rho_{\calY}(g) f_\theta(x)$ for all $g \in G$, this equivariance is strictly preserved by the augmented gradient flow. While this yields a conservation statement on this specific locus, it is a property of the model's inherent symmetry rather than a \emph{new} conservation law emerging on a \emph{generic, open} subset of the parameter space.

There is one more subtlety, namely the flexibility of the choice of $\rho_\calY(g)$. If $\varepsilon :G \to \{\pm 1\}$ is equivariant (a \emph{character}), $\rho_\calY\cdot \varepsilon$ is also a valid representation, and the set of $\theta$ yielding the symmetry $f_\theta\circ \rho_\calX(g) = \rho_\calY^{\varepsilon}\circ f$ are preserved in the same manner as just described.

 To exclude these obstructions, define for fixed input $x$ and parameters $\theta$ the quantity
\begin{equation}
\label{eq:chi_negative}
    \chi^{x,\theta}(g) =
    \rho_{\calY}(g)^{-1} f_\theta(\rho_{\calX}(g)x).
\end{equation}
For equivariant networks with respect to $\rho_\calY$($\rho_\calY^\varepsilon$) all values of $\chi^{x,\theta}$ coincide (up to sign, respectively). 
To exclude these cases, we make the following assumptions:
 \begin{assumption}
    \label{ass:negres}
There exist a dense subset $\calX_{\mathrm{reg}}\subseteq\calX$ and an open subset $\Theta_{\mathrm{reg}}\subseteq\Theta$ such that, for every $x\in\calX_{\mathrm{reg}}$ and $\theta\in\Theta_{\mathrm{reg}}$, the map $\chi^{x,\theta}:G\to\calY$ has no zero values and is injective modulo signs. The latter means for any $g \neq h$
\begin{equation}
    \chi^{x,\theta}(g) \neq \pm \chi^{x,\theta}(h).
\end{equation}
\end{assumption}

We can now state the following theorem. 
 \begin{theorem}
\label{thm:no_new_integrals_nonpolynomial}
     Assume that $G$ is finite and acts orthogonally on $\calX$ and $\calY$. Let $\ell$ be a margin loss of the form \eqref{eq:margin_loss_negative}, where $\mathcal L$ is analytic but not a polynomial. Under Assumption~\ref{ass:negres},
     for every $\theta\in\Theta_{\mathrm{reg}}$, the span of the augmented gradients matches the span of the standard empirical gradients:
     \begin{equation}
\label{eq:negative_span_equality}
\begin{aligned}
    \operatorname{span}_{(x,y)} \nabla \mathcal E^G_{(x,y)}(\theta) = \operatorname{span}_{(x,y)} \nabla \mathcal E_{(x,y)}(\theta).
\end{aligned}
\end{equation}
Consequently, every integral of motion on $\Theta_{\mathrm{reg}}$ for symmetric data is already an integral of motion for arbitrary data. 
 \end{theorem}
The proof is given in Appendix \ref{app:no_new_integrals_nonpolynomial}. The idea relies on the non-polynomial nature of the loss. More precisely, the augmented gradient is a finite sum over group orbits. Since $\mathcal L$ is analytic and non-polynomial, its Taylor series contains infinitely many nonzero terms. Under Assumption \ref{ass:negres}, these infinitely many varying powers act as a full-rank linear system, allowing us to isolate and separate the individual orbit gradients. Consequently, the group average does not collapse or remove any gradient directions.

\subsection{Conservation laws from lifted symmetries}
\label{subsec:more_integrals}

We now describe how  data augmentation can create new integrals of motion for a family of tensorizable networks trained with MSE loss. This case differs from the analytic non-polynomial losses discussed above. There, the Taylor expansion produces infinitely many algebraic conditions, which generically prevent the symmetry in the data from propagating to the parameter space. For MSE, however, the loss is determined only by its quadratic and linear terms. Consequently, these terms, after augmentation, impose only finitely many equivariance constraints on the lifted feature space; see \eqref{eq:Hs}. Thus, some continuous symmetries can survive the augmented MSE loss, and in our setting, they are captured by a, in general continuous, group $\mathsf H$ acting on the lifted feature space $\calH$. When this action can be pulled back to $\Theta$, new integrals of motion emerge.

Recall that for a tensorizable network, we have $f_\theta(x)=M(\theta)T(x)$. The MSE loss then lifts to the space of linear maps $M\in\Hom(\calH,\calY)$ via
\begin{equation}
    \calJ_\pi(M) = \mathbb{E}_\pi\left(\| MT(x)-y\|^2\right),
\end{equation}
where $\calJ_\pi$ relates to the risk $\calE_\pi$ through $\calE_\pi = \calJ_\pi \circ M$.
We also assume that the lift $T$ is compatible with the data symmetry i.e., there is an orthogonal representation $\rho_\calH$ of $G$ such that
\begin{equation}
\label{eq:Teq}
    T(\rho_\calX(g)x)=\rho_\calH(g)T(x),
    \qquad \forall g\in G.
\end{equation}

For MSE, the lifted augmented loss is determined by its quadratic and linear parts in $M$. After averaging over $G$, these parts become equivariant.  More precisely, the quadratic part belongs to $\Sym_G^2(\calH)$, the space of $G$-equivariant symmetric operators on $\calH$, while the linear part belongs to $\Hom_G(\calH,\calY)$, the space of $G$-equivariant linear maps from $\calH$ to $\calY$. The loss will consequently be invariant to any linear map on the lifted feature space that preserve all such equivariant quadratic and linear terms, i.e., to
\begin{equation}
\label{eq:Hs}
    \mathsf H = \bigg\{ B\in \operatorname{GL}(\calH)\bigg| \begin{tabular}{l} $BSB^\top =S$,  $S\in \Sym^2_G(\calH)$ \\
    $LB^\top = L$, $L \in \Hom_G(\calH,\calY)$ \end{tabular}  \bigg\}.
\end{equation}

\begin{prop}
\label{prop:Hs}
Under the equivariance assumption \eqref{eq:Teq}, the augmented lifted loss $\calJ_\pi^G$ is invariant under the right action of $\mathsf H$:
\begin{equation}
    \calJ_\pi^G(MB)=\calJ_\pi^G(M),
    \ \   M\in\Hom(\calH,\calY),\ B\in\mathsf H.
\end{equation}
\end{prop}
The proof is given in Appendix \ref{app:otherproofs}.

  The possible new symmetries are therefore controlled by $\mathsf H$, whose structure depends on $G$ as well as the representations $\rho_\calH$ and $\rho_\calY$. We briefly explain its structure in the following remark.
\begin{remark}
\label{rmk:H_structure}
Since the identity operator belongs to $\Sym^2_G(\calH)$, the condition $BSB^\top=S$ implies that every $B\in\mathsf H$ lies in $\O(\calH)$. Hence, for $B\in\mathsf H$, this condition is equivalent to the commutation relation $BS=SB$. For a fixed symmetric operator $S~\in~\Sym_G^2(\calH)$, any orthogonal matrix $B$ commuting with $S$ must preserve its eigenspaces. Thus, up to a change of basis, $B$ belongs to a product of orthogonal groups whose dimensions match the corresponding eigenvalue multiplicities of $S$.

However, our construction is data-independent. Hence, $B$ must simultaneously commute with all $S \in \Sym_G^2(\calH)$. Thus, $\mathsf H$ is the intersection of all such eigenspace-preserving subgroups. As explained in Appendix \ref{app:hchar}, the continuous symmetries that may survive this intersection form a product of orthogonal, unitary, or symplectic groups.
Finally, the linear constraints $LB^\top=L$ may remove some of these components, but continuous factors can remain, even when $G$ is finite.
\end{remark}

We now explain how such a lifted symmetry can produce a conservation law in parameter space. Suppose that the parameters contain a matrix block $P\in\Hom(E,V)$, and write $\theta=(P,\xi)$ for the remaining parameters. We say that the parametrization \emph{realizes} an $\O(V)$-symmetry inside $\mathsf H$ on $P$ if, for every $a\in \O(V)$, there exists $B_a\in\mathsf H$ such that
\begin{equation}
\label{eq:H-eq}
    M(aP,\xi)=M(P,\xi)B_a.
\end{equation}
This means that rotating the parameter block $P$ on the left is equivalent, after applying $M$, to acting on the lifted feature space by an element of $\mathsf H$. Combining \eqref{eq:H-eq} with Proposition \ref{prop:Hs} gives
\begin{equation}
\label{eq:invariant_block}
    \calE_\pi^G(aP,\xi)=\calE_\pi^G(P,\xi), \qquad a\in \O(V).
\end{equation}

In the next theorem, we see that if the symmetry is realized in that way, the range of $P$ is an integral of motion.

\begin{theorem}
\label{thm:realized_lifted_symmetry} 
Let $f_\theta(x)=M(\theta)T(x)$ be a tensorizable network trained with the augmented MSE loss. Assume that the parametrization realizes an $\O(V)$-symmetry inside $\mathsf H$ on the block $P\in\Hom(E,V)$, in the sense of \eqref{eq:H-eq}. Then, along the augmented gradient flow,
\begin{equation}
    \range P(t)=\range P(0).
\end{equation}
Consequently, every smooth real-valued function of the column space of $P$ is an integral of motion. 
\end{theorem}

We provide the proof in Appendix \ref{app:finite_sample_lifted_symmetries}. 
The main idea is that, by applying the first fundamental theorem of the orthogonal group \parencite{weyl1946classical} to the invariance in \eqref{eq:invariant_block}, the augmented risk depends on $P$ only through $P^\top P$, i.e., $\calE_\pi^G(P,\xi)=\Phi(P^\top P,\xi)$ for some function $\Phi$. Consequently, the $P$-component of the gradient flow has the form $\dot P=PS(P,\xi)$ for a suitable matrix $S(P,\xi)$. 
The fundamental theory of ODEs then implies that $\range P$ is necessarily preserved along the gradient flow.

\begin{remark}
\label{rmk:new_integrals}
Theorem \ref{thm:realized_lifted_symmetry} gives a new integral of motion only when the symmetry is created by augmentation. This happens when $\mathsf H$ contains a continuous symmetry that is not already present for the ordinary loss (see Figure \ref{fig:symmetric-loss-landscape}), and when the parametrization realizes this symmetry on some parameter block $P$ through \eqref{eq:H-eq}. The resulting integral is non-trivial when the column space of $P$ can vary, for instance when $0<\operatorname{rank}P<\dim V$.
\end{remark}
\section{EXAMPLES}
\label{sec:examples}
We have stated Theorem \ref{thm:realized_lifted_symmetry} in a generally applicable but abstract way. In this section, we will present two exemplary applications of it.

\subsection{Example 1: A linear model} \label{ex:linearmodel}
Let $\calX=\R^3$ and $\calY=\R$, and consider the linear model
\begin{align}
    f_W(x) = Wx.
\end{align}
This is a tensorizable network with $\calH=\calX$ and $T$ equal to the identity. We let the cyclic group $C_3=\{e,r,r^2\}$ act trivially on $\calY$ and cyclically on the standard basis vectors $e_i\in \calX$ by
\begin{equation}
\rho_\calX(r) e_i = e_{(i+1)\, \mathrm{mod} \, 3}.    
\end{equation}
We first compute the augmented loss for a single, but arbitrary, data point $(x,y)$. Let $\mathbf{1}$
denote the constant one-vector, and define
\begin{equation} 
\Pi_{\mathbf 1}=\frac{1}{3}\mathbf 1\mathbf 1^\top, \quad
\Pi_{\one^\perp}=I-\frac{1}{3}\mathbf 1\mathbf 1^\top .
\end{equation}
A straightforward calculation shows that after averaging over the $C_3$-orbit of $x$, the quadratic and linear terms in the MSE loss become
\begin{equation}
\frac{1}{3}\sum_{k=0}^2 \rho_\calX(r)^k x x^\top \rho_\calX(r)^{-k} = \alpha \Pi_{\mathbf 1}+\beta \Pi_{\one^\perp},
\end{equation}
and
\begin{equation}
\frac{1}{3}\sum_{k=0}^2 y\bigl(\rho_\calX(r)^kx\bigr)^\top = \gamma \mathbf 1^\top ,
\end{equation}
respectively. Here $\alpha,\beta$ and $\gamma$ are real numbers depending on $(x,y)$. Hence, the augmented gradient flow has the form
\begin{equation}
\label{eq:grad_for_linear}
\dot W = -2W(\alpha \Pi_{\mathbf 1}+\beta\Pi_{1^\perp})+2\gamma\mathbf 1^\top .
\end{equation}
Multiplying \eqref{eq:grad_for_linear} by $\Pi_{\one^\perp}$, utilizing that $\Pi_{\one^\perp}\one=0$ and $\Pi_1\Pi_{\one^\perp}=0$, gives
\begin{equation}
\frac{d}{dt}(W\Pi_{\one^\perp})=-2\beta W\Pi_{\one^\perp}.
\end{equation}
The derivative of $W\Pi_{\one^\perp}$ being parallel to its position vector implies that its direction is conserved.  In particular, on the open set where $w_1+w_2-2w_3 \neq 0$, the function
\begin{equation}
\label{eq:integ_linear}
    I(W) = \frac{w_1-w_2}{w_1+w_2-2w_3}
\end{equation}
is an integral of motion. Since the augmented risk of a dataset is the expectation of the single-point augmented losses, \eqref{eq:integ_linear} holds after averaging, with $\beta$ replaced by its expectation. Hence, $I$ is an integral of motion for every dataset augmented by the $C_3$ action.

The above computation gives a direct way to obtain integral of motion for this example. However, we can also use the the general framework of Section \ref{subsec:more_integrals}.  To calculate $\mathsf{H}$, note that $\Pi_1$ is a symmetric, equivariant map. Hence, any $B\in \mathsf{H}$ commutes with it, and must therefore respect the decomposition
\begin{equation}
\calX=\Span\mathbf 1\oplus\Span\mathbf 1^\perp
\end{equation}
Applying the same argument to the equivariant map $\one^\top$ shows that $B$ must also fix $\Span \mathbf 1$. The remaining freedom is an arbitrary orthogonal transformation on $\Span \mathbf{1}^\perp$. Hence, $\mathsf{H} = \O(\Span \mathbf{1}^\perp)$, acting  orthogonally on $\Span \mathbf 1^\perp$ and fixing $\Span \mathbf 1$.

Finally, this symmetry is realized in parameter space by writing
\begin{equation}
    W = P^\top + \xi \mathbf{1}^\top, \quad P \in \Span \mathbf{1}^\perp,
\end{equation}
where $\xi \in \R$.  For $a\in \O(\Span \mathbf{1}^\perp)$, extended by the identity on $\Span \mathbf 1$, we have
\begin{equation}
    M(aP,\xi) = M(P,\xi)a^{-1}.
\end{equation}
Theorem \ref{thm:realized_lifted_symmetry} therefore implies that the range of $P$ is preserved. Since $P= (W \Pi)^\top$, that range corresponds to the direction of $W\Pi$ found above.

\subsection{Example 2: Lightning Attention}
\label{sec:Lightning_att}
We now turn to the tensorized representation of the single-head Lightning Attention introduced in Example \ref{ex:tensorized_lightning}. Let the group $G= \O(d)$ act on the input space from the right via $X \mapsto Xg$, and assume the labels are invariant to this action i.e., $\rho_\calY(g)Y = Y$.

We first determine the induced action on the lifted feature space $\calH$. Since the rows of $Xg$ are given by $(g^\top x_i)^\top$ and the feature covariance transforms as $g^\top Z_X g$, the tensorized map $T$ satisfies
\begin{equation}
    T(Xg) = \rho_{\calH}(g)T(X),
\end{equation}
where $\rho_{\calH}(g)(v\otensor S) = g^\top v \otensor g^\top Sg$.

Next, consider a simultaneous left-action of $\O(d)$ on the parameter space as follows
\begin{equation} \label{eq:hmap}
    h\theta = (hQ, hK, hV), \qquad h \in \O(d).
\end{equation}
By the construction of the parameter map $M$, this action is compatible with the induced action on $\calH$. Indeed, evaluating $C_{h\theta}$ yields
\begin{align}\label{eq:Ceq}
    C_{h\theta}(v \otensor S) &= (hV)^\top S(hK)(hQ)^\top v \nonumber \\ &= V^\top (h^\top Sh)KQ^\top (h^\top v) \\
    &= C_\theta \bigl( \rho_{\calH}(h) (v \otensor S) \bigr). \nonumber
\end{align}
This implies that $M(h\theta) = M(\theta)\rho_\calH(h)$. Combining this with the equivariance of $T$, we find that acting on the parameter is equivalent to acting on the input:
\begin{equation}
\begin{split}
    f_{h\theta}(X) &= M(h\theta)T(X) \\
    &= M(\theta)\rho_{\calH}(h)T(X) \\
    &= M(\theta)T(Xh) = f_\theta(Xh).
\end{split}
\end{equation}
With the above relation, we can now show that the augmented loss is invariant under the left $\O(d)$-action of the input
\begin{align}
    \calE_\pi^G(h\theta) &= \int_{\O(d)} \erw_\pi\left(\|{f_{h\theta}(Xg)-Y}\|^2\right) \dd \mu(g) \nonumber \\
    &= \int_{\O(d)} \erw_\pi\left(\|{f_{\theta}(Xgh)-Y}\|^2\right) \dd \mu(g) \\
    &= \int_{\O(d)} \erw_\pi\left(\|{f_{\theta}(Xg')-Y}\|^2\right) \dd \mu(g') \nonumber \\
    &= \calE_\pi^G(\theta) \nonumber.
\end{align}

Now let $P = \begin{bmatrix} Q & K & V \end{bmatrix} \in \R^{d \times (2r+m)}$ denote the concatenated parameter block. Combining the $\O(d)$-invariance of the loss and the argument in Section \ref{subsec:more_integrals} we deduce that the column space of $P$ is preserved, i.e., the range of $P$ is an integral of motion. 
\begin{remark}
 If the orthogonal group acts instead on the left (token dimension), then Lightning Attention is already equivariant with respect to this action. In that case, augmentation does not introduce an additional symmetry of the loss; it only reflects a symmetry already present in the architecture. Hence, we should not expect new integrals of motion from this token-side augmentation.
\end{remark}

\section{ILLUSTRATIVE EXPERIMENTS}
\label{sec:illustrative_experiments}
We end the paper by presenting two qualitative experiments for the conservation laws described in Section \ref{sec:examples}. These 
test the robustness of the results to two necessary deviations from theory in practice: using (full batch, fixed learning rate) gradient descent instead of gradient flow, and discretizing an infinite group average by finitely many samples. The integrals of motion should therefore not be expected to be \emph{perfectly} preserved, but  they should \emph{approximately}.

\subsection{The linear model}
We first consider the linear model from Section \ref{ex:linearmodel}. We draw four random data points in $\R^3$ with random scalar labels, and augment each point by its full $C_3$-orbit, giving twelve training points in total. We compare full-batch gradient descent on the standard and augmented MSE losses for six training runs with step size $1.5\cdot 10^{-2}$.

The true augmented dynamics preserves the direction of $W\Pi_{\one^\perp}$. As mentioned in Section \ref{ex:linearmodel}, this gives rise to the integral of motion $I(W)$ in \eqref{eq:integ_linear}.
In Figure \ref{fig:linear_integral_drift_new}, we draw the deviation $|I(W_t)-I(W_0)|$  for $150$ steps. For the augmented data, this quantity remains (almost) at numerical precision. In contrast, for the non-augmented case, it changes by several orders of magnitude.

\begin{figure}[t]
    \centering
    \includegraphics[width=\columnwidth]{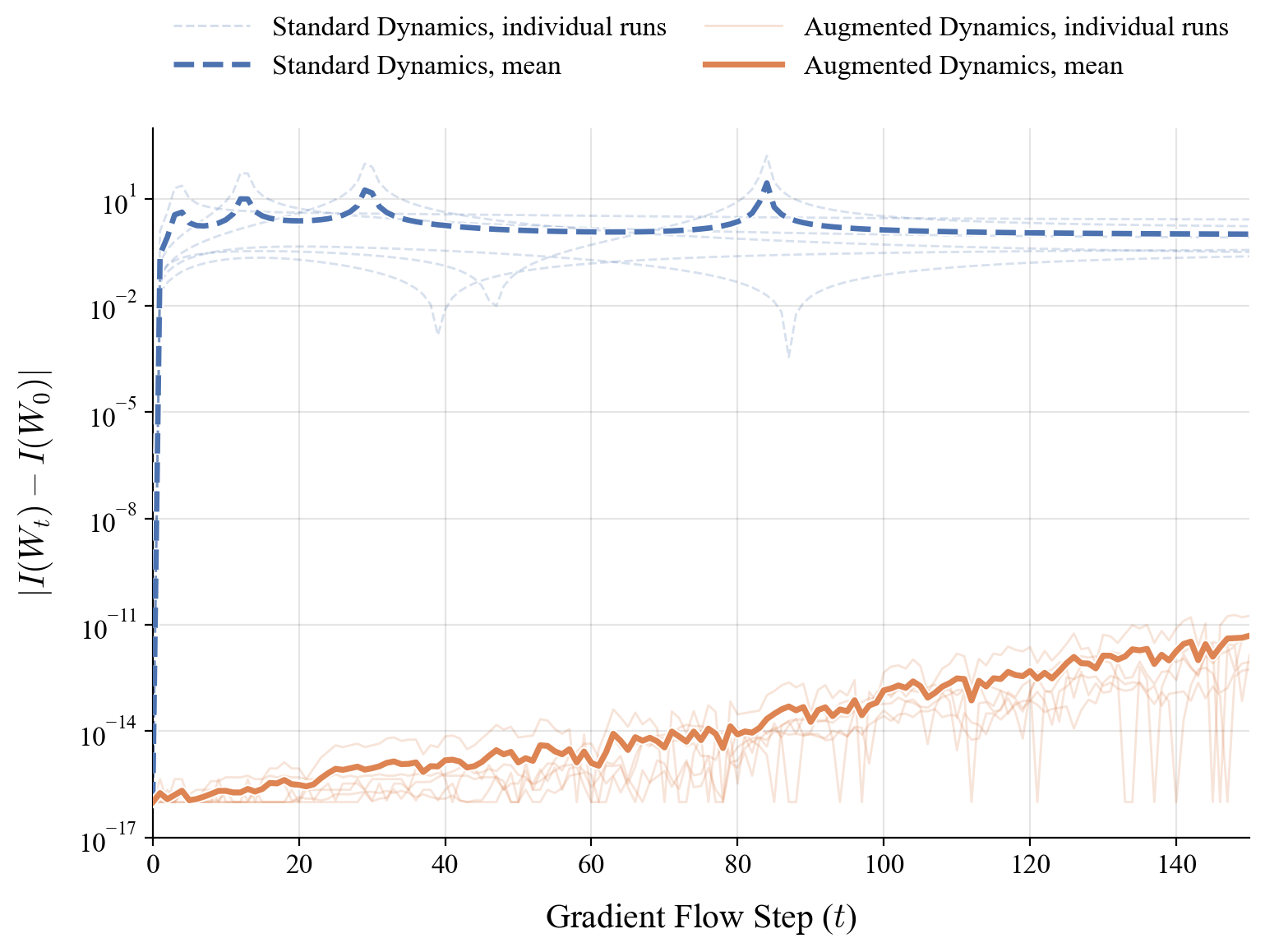}
    \caption{
    Evolution of the deviation $|I(W_t)-I(W_0)|$ over gradient descent steps for the linear model with $C_3$-augmentation. 
    }
    \label{fig:linear_integral_drift_new}
\end{figure}

\subsection{The Lightning Attention}
We next consider the single-head Lightning Attention model from Section \ref{sec:Lightning_att} with $d=8$, $n=5$, and $r=m=1$. The labels are generated from a randomly initialized teacher model on ten base inputs. We compare full-batch gradient descent on the standard MSE loss with full-batch gradient descent on an augmented MSE loss for the right $\O(d)$-action $X\mapsto Xg$. The Haar integral is approximated using $4096$ Haar-sampled orthogonal matrices, sampled by QR decomposition of Gaussian matrices; see \parencite{mezzadri2006generate}. We include both $g$ and $-g$ for each sampled matrix to reduce the finite-sample error of the group average. We use six full-rank initializations of $P=[Q\ K\ V]$ and step size $10^{-5}$.

The theory predicts that the column space of $P$ is preserved by the exact augmented gradient flow. We therefore track $\|\Pi_{\range(P_t)}-\Pi_{\range(P_0)}\|_F$, where $\Pi_{\range(P)}$ denotes the orthogonal projector onto $\range(P)$; this is, up to a constant factor, the chordal distance on the Grassmannian \parencite{conway1996packing}. Figure \ref{fig:lightning_projector_drift} shows that as expected, this quantity does not stay exactly zero in this finite Haar-sampling setting, but remains several orders of magnitude smaller than for the non-augmented dynamics. 

\begin{figure}[t]
    \centering
    \includegraphics[width=\columnwidth]{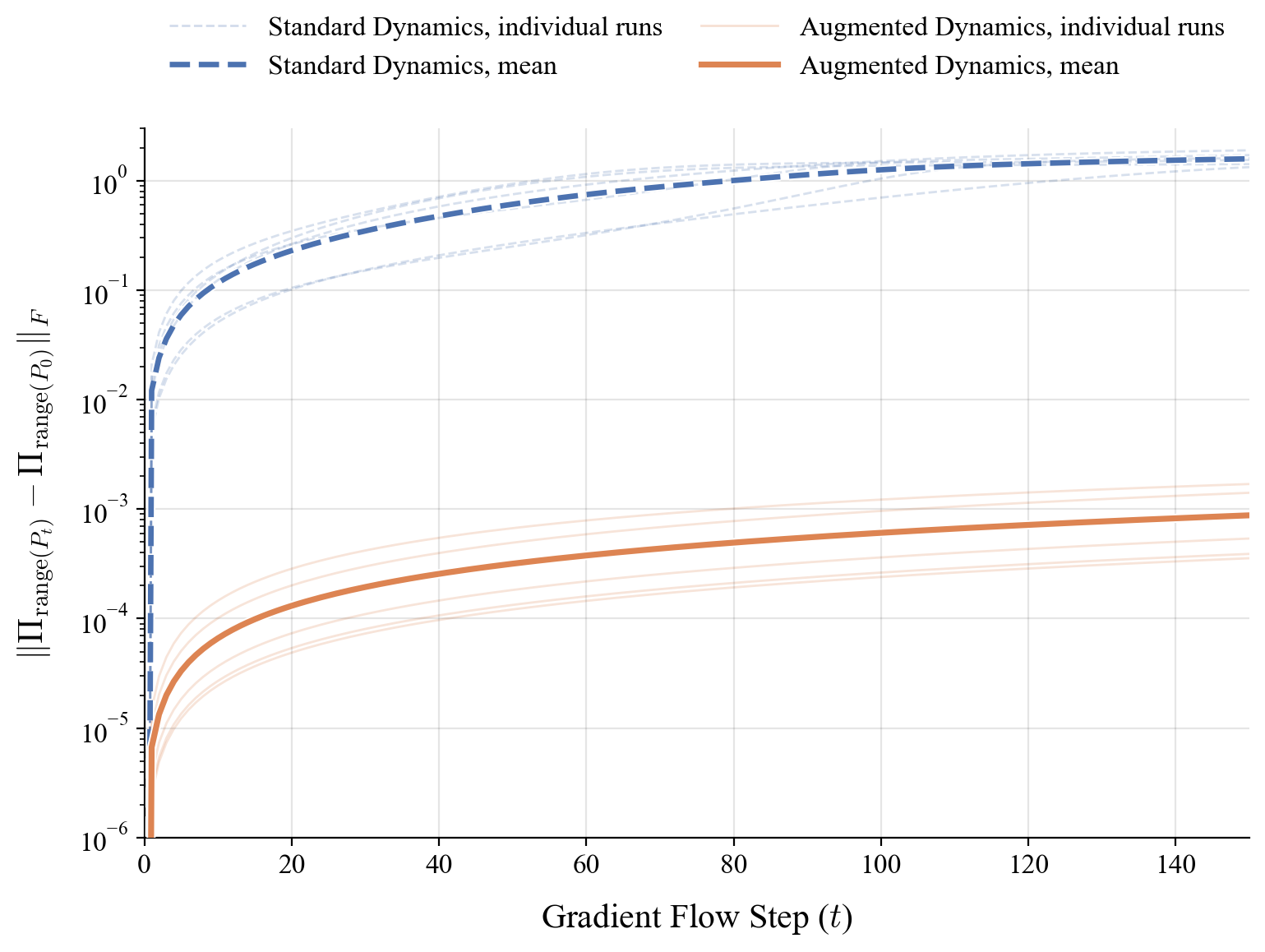}
    \caption{
    Evolution of the column-space distance $|\Pi_{\range(P_t)}-\Pi_{\range(P_0)}\|_F$ 
    over gradient descent steps for single-head Lightning Attention with right $\O(d)$-augmentation.
    }
    \label{fig:lightning_projector_drift}
\end{figure}
\section{CONCLUSION}
We investigated whether data symmetries can give rise to new conserved quantities under gradient-flow training. Our theory shows that the existence of such quantities depends on the loss function, the structure of the augmentation group, and the architecture.

We derived two main theoretical results. First, with a finite augmentation group and a pure analytic margin-type loss, new conserved quantities cannot arise, except in rare singular cases. 
The second result concerns MSE loss, which due to its polynomial nature behaves differently. For \emph{tensorizable} neural networks, we showed that the augmented loss may contain additional, possibly continuous, lifted symmetries, which can correspond to conserved quantities in parameter space. We verified our findings with small numerical experiments,  showing that the integrals of motion are approximately preserved also when the dynamics and infinite groups are discretized.

The limitations of our work are the needed technical assumptions. In the first result, the study is limited to analytic margin-type losses and smooth parametrizations.
For the second result, not all architectures are tensorizable, and not all $\mathsf{H}$-symmetries are realizable. We intend to address all these in future work.

There are several more open questions. First, we do not claim that the framework covers all possible new conserved quantities from symmetric data. On the other hand, there might be other realizations of the $\mathsf{H}$ group than the orthogonal one discussed here. Another important direction is to study more examples of tensorizable networks.

\section*{Acknowledgements}
This work was supported by the Wallenberg AI, Autonomous Systems and Software Program (WASP) funded by the Knut and Alice Wallenberg Foundation.

\section*{Declaration of AI-use}
 We used GPT 5.5 (Extended) to generate code for the numerical experiment (that was subsequently reviewed). 
 
\printbibliography

\newpage

\appendix 

\onecolumn

\section{PROOF OF THEOREM~\ref{thm:no_new_integrals_nonpolynomial}}
\label{app:no_new_integrals_nonpolynomial}
Our goal is to prove the span equality in \eqref{eq:negative_span_equality}. The inclusion $\operatorname{span} \nabla \calE^G \subseteq \operatorname{span} \nabla \calE$ is an immediate consequence of the augmented risk definition, as $\nabla \mathcal E^G$ is a convex combination of standard gradients over the group orbit. The challenge is proving the reverse inclusion i.e., showing that averaging over the group does not destroy any gradient directions. We do this by showing that the individual, non-averaged gradient directions can be linearly reconstructed from the augmented gradients.

To this end, let us fix $\theta\in\Theta_{\mathrm{reg}}$ and $x\in\calX_{\mathrm{reg}}$. Also, for any $g\in G$, we denote the network's prediction on the transformed input, pulled back to the original space, as 
\begin{equation}
\label{eq:ug_appendix}
    u_g := \rho_{\calY}(g)^{-1} f_\theta(\rho_{\calX}(g)x) \in\calY.
\end{equation}
We also define the corresponding Jacobian of this transformation with respect to the parameters
\begin{equation}
\label{eq:Jg_appendix}
    J_g := D_\theta \Bigl[ \rho_{\calY}(g)^{-1} f_\theta(\rho_{\calX}(g)x) \Bigr].
\end{equation}

Applying the chain rule to the margin loss \eqref{eq:margin_loss_negative}, the gradient for a single transformed data point is
\begin{equation}
\label{eq:single_orbit_gradient_appendix}
    \nabla_\theta \mathcal E_{(\rho_{\calX}(g)x,\rho_{\calY}(g)y)}(\theta) = \mathcal L'(\langle u_g,y\rangle) J_g^*y,
\end{equation}
where $J_g^*:\calY\to T_\theta\Theta$ denotes the Euclidean adjoint. Because of the orthogonality on the group action on $\calY$, the inner product inside the margin loss simplifies to $\langle f_\theta(\rho_{\calX}(g)x), \rho_{\calY}(g)y \rangle = \langle u_g,y\rangle$.

By averaging this over the finite group $G$, we get
\begin{equation}
\label{eq:augmented_gradient_appendix}
    \nabla_\theta \mathcal E^G_{(x,y)}(\theta) = \frac{1}{|G|} \sum_{g\in G} \mathcal L'(\langle u_g,y\rangle) J_g^*y.
\end{equation}

We now first verify that restricting our inputs to the dense regular set $\calX_{\mathrm{reg}}$ does not artificially restrict the gradient span.
\begin{lemma}
    \label{lem:dense_inputs_same_span}
Let $\calX_{\mathrm{reg}}\subseteq\calX$ be a dense set in Euclidean topology. Then, for every $\theta\in\Theta$,
\begin{equation}
\label{eq:span_loss_lemm}
\operatorname{span}_{x\in\calX_{\mathrm{reg}},\,y\in\calY} \nabla \mathcal E_{(x,y)}(\theta) = \operatorname{span}_{x\in\calX,\,y\in\calY} \nabla \mathcal E_{(x,y)}(\theta).
\end{equation}
\end{lemma}
\begin{proof}
    The inclusion $\subseteq$ is trivial. For the reverse, choose any $x_0\in\calX$ and $y\in\calY$. Because $\calX_{\mathrm{reg}}$ is dense, there exists a sequence $x_n\in\calX_{\mathrm{reg}}$ converging to $x_0$. By the smoothness of the parametrization and the loss, we have that $\nabla \mathcal E_{(x_n,y)}(\theta)$ converges to $\nabla \mathcal E_{(x_0,y)}(\theta)$. Since the span on the left-hand side of \eqref{eq:span_loss_lemm} is a subspace of the finite-dimensional tangent space $T_\theta\Theta$, it is Euclidean closed.  Therefore, it must contain the limit point $\nabla \mathcal E_{(x_0,y)}(\theta)$, which finishes the proof.
\end{proof}

Next, we show that under Assumption \ref{ass:negres}, the vectors $u_g$ are distinct modulo signs. We aim to find a projection vector $v$ that maintains this distinctness, allowing us to construct an invertible generalized Vandermonde matrix from the infinite Taylor series of $\mathcal{L}$.
\begin{lemma}
\label{lem:orbit_separation}
Let $\{u_g\}_{g\in G}\subseteq\calY$ satisfy $u_g\neq0$ and $u_g\neq\pm u_h$ for all $g\neq h$.  Then there exists a vector $v\in\calY$ such that for $\lambda_g=\langle u_g,v\rangle$, the scalars $|\lambda_g|$ are non-zero and pairwise distinct. Furthermore, let $A\sse \N$ be an arbitrary infinite subset. There then exist $n_1,\ldots,n_{|G|}\in A$ such that the matrix
\begin{equation}
\label{eq:vandermonde_appendix}
    \bigl(\lambda_g^{\,n_j}\bigr)_{1\leq j\leq |G|,\;g\in G}
\end{equation}
is invertible.
\end{lemma}
\begin{proof}
    For any $g\neq h$, the equality $|\langle u_g,v\rangle| = |\langle u_h,v\rangle|$ holds only if $v$ lies on one of the two hyperplanes defined by $\langle u_g-u_h,v\rangle=0$ or $\langle u_g+u_h,v\rangle=0$. Because $u_g\neq\pm u_h$, these are proper hyperplanes. Similarly, $\langle u_g,v\rangle=0$ defines a proper hyperplane since $u_g\neq0$. Now, note that a finite-dimensional vector space cannot be covered by a finite union of proper hyperplanes, thus we can simply choose $v$ outside this union, ensuring the $|\lambda_g|$ are distinct and non-zero.

    To prove the second claim, assume for contradiction that no such subset of exponents yields an invertible matrix. This implies that the vectors $\bigl(\lambda_g^n\bigr)_{g\in G}$ for $n\in A$ fail to span $\mathbb R^{|G|}$. Thus, there exists coefficients $c_g$, not all zero, such that $\sum_{g\in G} c_g\lambda_g^n=0$ for every $n\in A$. Let $h$ be an index where $c_h\neq0$ and $|\lambda_h|$ is maximal among all indices with non-zero coefficients. Note then that since $\lambda_h\neq \pm \lambda_g$ for all $g\neq h$, we must have $|\lambda_g|<|\lambda_h|$ for all $g\neq h$. Dividing the sum by $\lambda_h^n$ yields 
    \begin{equation}
    c_h + \sum_{g\neq h} c_g \left( \frac{\lambda_g}{\lambda_h} \right)^n = 0 \qquad \text{for every } n\in A .
\end{equation}
Because $A$ is infinite, it contains arbitrarily large integers. As we take $n\to\infty$ along $A$, the fractional terms vanish due to $|\lambda_g|/|\lambda_h|<1$ for $g\neq h$. This leaves $c_h=0$, contradicting our initial assumption.
\end{proof}

 We now by using the Taylor expansion of the non-polynomial loss, we can extract the individual orbit directions $J_h^*y$ from the augmented span.

 \begin{lemma}
\label{lem:orbit_terms_in_augmented_span}
Let $\theta\in\Theta_{\mathrm{reg}}$ and $x\in\calX_{\mathrm{reg}}$, then for every $h\in G$ and every $y\in\calY$, we have
\begin{equation}
    J_h^*y \in \operatorname{span}_{y'\in\calY} \nabla_\theta \mathcal E^G_{(x,y')}(\theta).
\end{equation}
\end{lemma}
\begin{proof}
    Let $\mathcal S = \operatorname{span}_{y'\in\calY} \nabla_\theta \mathcal E^G_{(x,y')}(\theta)$, and define the function $F(y) = \nabla_\theta \mathcal E^G_{(x,y)}(\theta)$. By definition, $F(y)\in \mathcal S$ for all $y$. 

    Since $\mathcal L$ is analytic and non-polynomial, its derivative can locally be written 
    \begin{equation} \label{eq:infseries}
    \mathcal L'(t) = \sum_{n\geq0} a_n t^n,
    \end{equation}
where the Taylor series has an infinite number of non-zero coefficients. Let $A = \{n\geq0 \, \vert \, a_n\neq0\}$ be this infinite set of indices.

    By Assumption \ref{ass:negres}, our vectors $u_g$ satisfy the conditions of Lemma \ref{lem:orbit_separation}. Let $v$ be the vector guaranteed by that lemma, and let $\lambda_g=\langle u_g,v\rangle$. We claim that for every $g\in G$,
    \begin{equation}
        J^*_gv \in \calS
    \end{equation}
    To show this, substitute a scaled target $y=tv$ into our augmented gradient formula \eqref{eq:augmented_gradient_appendix} to get
\begin{equation}
    F(tv) = \frac{1}{|G|} \sum_{g\in G} t\,\mathcal L'(t\lambda_g)J_g^*v  \stackrel{\eqref{eq:infseries}}{=}   \sum_{n\geq 0}a_n t^{n+1} \left(\frac{1}{|G|}\sum_{g\in G}\lambda_g^n\,J_g^*v\right).
\end{equation}
The curve $t\mapsto F(tv)$ lies within the finite-dimensional subspace $\mathcal S$. Consequently, all Taylor coefficients of this curve evaluated at $t=0$ must also belong to $\mathcal S$. Extracting the coefficient for $t^{n+1}$ (where $n\in A$) yields that there exist $s_n\in \calS$, $n\in A$, with
\begin{equation}
\label{eq:Jgv_combinations_appendix}
    \sum_{g\in G} \lambda_g^n J_g^*v =s_n , \quad n \in A.
\end{equation}
But by Lemma \ref{lem:orbit_separation}, we can select a subset of exponents $B\sse A$ so that $(\lambda_g^n)_{n\in B,g\in G}$ has an inverse, say $(\mu_{g,n})_{g\in G, n\in B}$. Applying this matrix to \eqref{eq:Jgv_combinations_appendix} yields
\begin{align*}
    J_g^*v = \sum_{n\in B} \mu_{g,n}s_n \in \calS
\end{align*}
for every $g\in G$, which was the intermediate claim.   

To generalize this to an arbitrary direction $y\in\calY$,
we consider the directional derivative at $tv$ in direction $y$, $DF_{tv}[y]$. Since $F$ maps everywhere into $\mathcal{S}$, this derivative must also lie in $\mathcal S$. This basically means, we need to evaluate the derivative with respect to a scalar $s$ at $s=0$
\begin{equation}
    DF_{tv}[y] =\left. \frac{d}{ds}F(tv+sy) \right |_{s=0} = \left. \frac{d}{ds} \left ( \frac{1}{|G|} \sum_{g \in G} \calL'(\langle u_g,tv+sy\rangle) J_g^*(tv+sy)\right) \right |_{s=0}.
\end{equation}
Since $s$ appears in both the scalar and the vector component, we apply the product rule. That is, we first differentiate the linear vector component $J_g^*(tv+sy)$ with respect to $s$ yields $J_g^*y$; this multiplies the un-differentiated scalar component evaluated at $s=0$. Second, we apply the chain rule to the scalar component $\calL'(\langle u_g,tv+sy\rangle)$. That pulls out the inner product $\langle u_g,y\rangle$; this multiplies the un-differentiated vector component evaluated at $s=0$, which is $tJ_{g}^*v$. Adding these two parts yields
\begin{equation}
    D F_{tv}[y] = \frac{1}{|G|} \sum_{g\in G} \mathcal L'(t\lambda_g)J_g^*y + \frac{1}{|G|} \sum_{g\in G} t\,\mathcal L''(t\lambda_g) \langle u_g,y\rangle J_g^*v .
\end{equation}

Notice that the second summation consists entirely of terms proportional to $J_g^*v$, which we have already proven lie in $\mathcal S$. Therefore, to ensure the entire expression remains in $\mathcal S$, the first summation must independently belong to $\mathcal S$. Taking the $t^n$ coefficients (for $n\in A$) of this first sum gives
\begin{equation}
\label{eq:Jgy_combinations_appendix}
    \sum_{g\in G} \lambda_g^n J_g^*y \in \mathcal S \qquad \text{for every } n\in A .
\end{equation}
By applying the same argument as before, we get $J_h^*y \in \mathcal S$, completing the proof.
\end{proof}

We are now ready to finalize the proof of Theorem \ref{thm:no_new_integrals_nonpolynomial}.
\begin{proof}[Proof of Theorem \ref{thm:no_new_integrals_nonpolynomial}]
    As mentioned initially, the inclusion $\operatorname{span} \nabla \mathcal E^G \subseteq \operatorname{span} \nabla \mathcal E$ is trivial. To prove the reverse inclusion, Lemma \ref{lem:dense_inputs_same_span} allows us to restrict to $x\in\calX_{\mathrm{reg}}$. For any such $x$, evaluating the standard empirical gradient \eqref{eq:single_orbit_gradient_appendix} at the identity element $g=e$ gives
\begin{equation}
    \nabla \mathcal E_{(x,y)}(\theta) = \mathcal L'(\langle u_e,y\rangle)J_e^*y .
\end{equation}
By Lemma \ref{lem:orbit_terms_in_augmented_span}, the vector $J_e^*y$ is strictly contained within the augmented span $\operatorname{span}_{y'\in\calY} \nabla_\theta \mathcal E^G_{(x,y')}(\theta)$. Consequently, $\nabla \calE_{(x,y)}(\theta)$ is also contained in the augmented span, proving the strict span equality \eqref{eq:negative_span_equality}.

Finally, suppose $I:\Theta_{\mathrm{reg}}\to\mathbb R$ is an integral of motion for symmetric data. By \eqref{eq:integral_grad}, its gradient $\nabla I(\theta)$ must be orthogonal to the augmented span. Because the augmented span equals the standard empirical span, $\nabla I(\theta)$ is necessarily orthogonal to all ordinary gradient directions as well. Therefore, $I$ is an integral of motion for arbitrary data, and no new integral of motion emerges.
\end{proof}
\newpage
\section{PROOF OF THEOREM \ref{thm:realized_lifted_symmetry}}
\label{app:finite_sample_lifted_symmetries}

In this section, we prove Theorem~\ref{thm:realized_lifted_symmetry}. We recall that the augmented loss over the parameter space is $\calE_\pi^G(\theta)=\calJ_\pi^G(M(\theta))$ and the parameter $\theta$ admits the decomposition $(P,\xi)$ with $P \in \Hom(E,V)$. Moreover, for every $a \in \O(V)$, there exists $B_a \in \mathsf H$, such that $M(aP,\xi) = M(P,\xi)B_a$. 

\begin{proof}
    Let $a\in \O(V)$. Since $B_a\in\mathsf H$, by definition of $\mathsf H$, we have $\calJ_\pi^G(MB_a)=\calJ_\pi^G(M)$ for all $M\in\Hom(\calH,\calY)$. Applying this to $M=M(P,\xi)$, we have
\begin{align}
    \calE_\pi^G(aP,\xi)
    &=\calJ_\pi^G(M(aP,\xi))   \\
    &=\calJ_\pi^G(M(P,\xi)B_a) && \qquad (\text{Equivariance of $M$})\\
    &=\calJ_\pi^G(M(P,\xi))  && \qquad{(B_a\in \mathsf H)} \\
    &=\calE_\pi^G(P,\xi).
\end{align}
Hence the augmented MSE loss is invariant under the left action on $P$ by $\O(V)$. This implies that the function can be written as 
\begin{equation}
    \label{eq:Phi_PPT}
    \calE_\pi^G(P,\xi) = \Phi(P^* P, \xi)
\end{equation}
for some smooth function $\Phi$. This is essentially the first fundamental theorem for the orthogonal group (see e.g. \parencite{weyl1946classical}, or \parencite{villar2021scalars} for a treatise connected to geometric deep learning). We give a short argument here for convenience:

What needs to be shown is that the left $\O(V)$-orbit in $\Hom(E,V)$ are determined by the Gram matrix $P^* P$. First, it is clear that the $P^* P$ needs to be conserved on those orbits:  $(aP)^* aP = P^* a^* aP=P^* P$ for $a\in \O(V)$. On the other hand, if $P$ and $\widetilde{P} \in \Hom(E,V)$ satisfy $P^* P=\widetilde{P}^* \widetilde{P}$, they must have the same rank. Therefore, there exists a unique homomorphism $\phi: \range P \to \range \widetilde{P}$ with that sends the $i$:th column of $P$ to the $i$:th column of $\widetilde{P}$, i.e. $\phi p_i = \widetilde{p}_i$. Since 
\begin{equation}
    \sprod{p_i,p_j} = (P^* P)_{ij} = (\widetilde{P}^* \widetilde{P})_{ij} = \sprod{\widetilde{p}_i,\widetilde{p}_j} = \sprod{\phi p_i, \phi p_j}
\end{equation}
for every $i,j$, $\phi$ must be an isometry. Extending this isometry to an element $a\in \O(V)$, we get $\widetilde{P}=aP$.

 Next, we argue that \eqref{eq:Phi_PPT} implies that
 \begin{equation} \label{eq:gradient}
     \nabla_P\calE_\pi^G(P,\xi)=P S(P,\xi)
\end{equation}
for some map $S(P,\xi)$  with values in the space $\Sym^2(E)$ of symmetric operators on $E$. To see this, let us declutter the notation by dropping the dependence on $\xi$. 
 Differentiating \eqref{eq:Phi_PPT} with respect to $P$, the chain rule yields
\begin{equation}
    \sprod{ \nabla_P \calE_\pi^G(P,\xi), \Delta P} = \sprod{\nabla_{P^* P}\Phi(P^* P),P^* \Delta P + \Delta P^* P} \\
    = \sprod{P(\nabla _{P^* P}\Phi(P^* P) + \nabla _{P^* P}\Phi(P^* P)^* ),\Delta P }, 
\end{equation}
i.e., reintroducing $\xi$, defining $\Sigma_P=P^*P$ and $S(P,\xi) =\nabla _{\Sigma_P}\Phi(\Sigma_P,\xi) + \nabla _{\Sigma_P}\Phi(\Sigma_P,\xi)^* \in \Sym^2(E)$, we get \eqref{eq:gradient}.

Hence, the $P$ component of augmented gradient flow satisfies $\dot P=-P S(P,\xi)$. This immediately shows that $P(t)=P(0)R(t)$, where $R(t)$ is the solution of the following ODE problem
\begin{equation}
    \dot R(t) = -R(t)S(P(t),\xi(t)), \qquad R(0) = I_E.
\end{equation}
Due to the theory of linear ODEs, $R(t)$ is invertible for all $t$ for which the flow exists. Hence, $P(t) = P(0)R(t)$ for an invertible $R(t)$, which implies $\range P(t)=\range P(0)$ (because such right multiplication preserves the range).

\end{proof}

\newpage

\section{OTHER OMITTED PROOFS} \label{app:otherproofs}
\begin{proof}[Proof of Proposition \ref{prop:Hs}]
    The 'Pythagorean formula' on $\calY$ and the definition of the inner product on the spaces of operators implies that
    \begin{align*}
        \calJ_\pi(
        M) &= \erw_\pi\left(\sprod{MT(x),MT(x)}\right) - 2 \erw_\pi(\sprod{MT(x),y}) + \erw_\pi(\sprod{y,y}) \\
        &= \erw_\pi(\sprod{MT(x),MT(x)}) - 2 \erw_\pi(\sprod{MT(x),y}) + \erw_\pi(\sprod{y,y}) \\
        &= \erw_\pi\left(\sprod{M^*M,T(x)T(x)^*}\right) - 2 \erw_\pi(\sprod{M,yT(x)^*}) + \erw_\pi(\sprod{y,y}) \\
        & = \sprod{M^*M, \erw_\pi(T(x)T(x)^*)} - 2 \sprod{M,\erw_\pi(yT(x)^*)} + \erw_\pi \left(\|y\|^2\right).
    \end{align*}
    The term $\|y\|^2$ is invariant under the orthogonal action $\rho_\calY$, and thus not affected by taking the $G$-average. Now let us investigate how the symmetric operator $S(x) = \erw_\pi\left(T(x)T(x)^*\right)\in \Sym^2(\calH)$ and the operator $L(y,x)=\erw_\pi(yT(x)^*
    )\in \Hom(\calH,\calY)$ are affected by symmetrization over $G$. The equivariance of $T$ implies
    \begin{align*}
        \overline{S} &:= \int_GS(\rho_\calX(g)x) \, \dd \mu(g) =  \int_GT(\rho_\calX(g)x)T(\rho_\calX(g)x)^* \, \dd \mu(g) = \int_G\rho_\calH(g)T(x)T(x)^*\rho_\calH(g)^* \, \dd \mu(g) \\
        \overline{L} &:=  \int_G L(\rho_\calY(g)y,\rho_\calX(g)x) \, \dd \mu(g) =  \int_G\rho_\calY(g)yT(\rho_\calX(g)x)^* \, \dd \mu(g) = \int_G\rho_\calY(g)yT(x)^*\rho_\calH(g)^* \, \dd \mu(g). \\
    \end{align*}
    It is clear that $\overline{S}$ is still symmetric. Additionally, $\overline{S}$ and $\overline{L}$ are equivariant. This property of the Haar average is very well-known -- it follows from the invariance of the Haar measure. For the reader's convenience, let us carry it out for $\overline{L}$. For $k\in G$ arbitrary, we have
    \begin{align*}
        \overline{L}\rho_\calH(k) &= \int_G\rho_\calY(g)yT(x)^*\rho_\calH(g)^*\rho_\calH(k) \, \dd \mu(g) = \int_G\rho_\calY(g)yT(x)^*\rho_\calH(k^{-1}g)^* \, \dd \mu(g) = \lceil k^{-1}g=h\rceil \\
        &= \int_G\rho_\calY(kh)yT(x)^*\rho_\calH(h)^* \, \dd \mu(h) = \rho_\calY(k)\int_G\rho_\calY(h)yT(x)^*\rho_\calH(h)^* \, \dd \mu(h) = \rho_\calY(k)\overline{L}.
    \end{align*}
    We conclude that
    \begin{align*}
        \calJ_\pi^G(M) = \sprod{M^*M,\overline{S}} -2\sprod{M,\overline{L}} + \erw_\pi\left(\| y \|^2\right)
    \end{align*}
    for an $\overline{S}\in \Sym^2_G(\calH)$ and an $\overline{L}\in \Hom_G(\calH,\calY)$. 
    Therefore, if $B\in \mathsf H$, we obtain
    \begin{align*}
        \calJ_\pi^G(MB) =  \sprod{B^*M^*MB,\overline{S}} -2\sprod{MB,\overline{L}} + \erw_\pi \left(\|y\|^2\right) = 
     \sprod{M^*M,B\overline{S}B^*} -2\sprod{M,\overline{L}B^*} + \erw_\pi \left(\|y\|^2\right) = \calJ_\pi^G(M),
    \end{align*}
    which is the claim.
\end{proof}

\begin{proof}[Proof of Proposition \ref{prop:augflow}]
(i) Let $I$ be a function that is preserved under the flow of all group augmented risks, and $\pi$ be a symmetric representations. For such $\pi$, we have $\calE_\pi=\calE_\pi^G$. Therefore, $I$ is preserved under the flow generated by $\calE_\pi$. It follows that $I$ is an integral of motion for symmetric data. 

    (ii) Now let $I$ be an integral of motion for symmetric data and $\pi$ be arbitrary. Consider the symmetrization $\pi_G$ of $\pi$: if $g\sim \mu$ and $(x,y)\sim \pi$, $(\rho(g)x,\rho(g)y)\sim \pi_G$. Then, $\pi_G$ is symmetric, and hence, $I$ is preserved under the flow 
    \begin{equation}
        \dot{\theta} = -\nabla \calE_{\pi_G}(\theta) =- \nabla \calE_\pi^G,
    \end{equation}
    where the last line follows from  $\calE_{\pi_G}= \calE_\pi^G$. This was what to be shown.    
\end{proof}

\newpage
\section{\texorpdfstring{A COMPLETE CHARACTERIZATION OF $\mathsf H$}{A COMPLETE CHARACTERIZATION OF H}} \label{app:hchar}
As mentioned in the main paper, we here want to give a complete characterization of the group $\mathsf H$. We will see that its structure tightly connects to the representation $\rho_\calH$, and in particular how it reduces $\calH$ into \emph{irreps}. To make things convenient for reader, we will not assume deep prior knowledge of real representation theory,  but instead keep the exposition self-contained, and present the necessary math in quite some detail. Readers well-versed in real representation theory can proceed directly to Section \ref{subsec:OrthogonalStabilizers} should they wish.

To soften up the notation a bit, let us drop the subindex of the representation $\rho_\calH$. Also, let us notice that $\rho$ without loss of generality can be assumed to be an orthogonal representation, $\rho(g)\in \O(\calH)$ for all $g\in G$. Indeed, we can redefine our inner product as
\begin{equation*}
    \sprod{x, y}_G := \int_G\sprod{\rho(g)x, \rho(g)y} \ \dd\mu(g)
\end{equation*}
where $\mu$ is the normalized Haar measure of $G$.

\subsection{Maschke's Theorem and Schur's Lemma}
Let $W$ be a finite-dimensional Hilbert space on which a compact group $G$ acts through an orthogonal representation $\rho$. It is common to refer to $W$ itself as a representation, and we will do so. An invariant space, or \emph{subrepresentation} of $\rho$, is a subspace $W'\subseteq W$  with $\rho(g)W'\sse W'$ for all $g\in G$. This of course implies that $G$ acts on $W'$ through the same representation: Formally, this representation is the map $\rho|_{W'}: G\to GL(W')$ defined by  $\rho(g)|_{W'} = \rho|_{W'}(g)$.
A representation $(\rho, W)$ is called \emph{irreducible} if the only subrepresentations are trivial. That is, if the only invariant subspaces are $W$ and $\{0\}$.

A non-trivial irreducible subrepresentation is called an \emph{irrep}. These are heavily studied because  representations often decompose into a direct sum of irreps.
That this is the case in our setting, with a compact group $G$, is \emph{Maschke's~theorem}.

\begin{theorem}(Maschke's Theorem) \parencite[Sec. 1.2]{fulton2013representation}
    Let $G$ be a compact Hausdorff group and $W$ be a finite dimensional representation of $G$.
    Then, 
    the representation of $G$ factors into an orthogonal product of non trivial irreducible representations.
\end{theorem}
Maschke's Theorem is usually stated for finite groups $G$. Let us therefore quickly present the standard proof to showcase that it also works for compact ones.
\begin{proof}
    Let $U\subseteq W$ be a non-trivial sub representation of $W$. We show that also $U^\perp$ is $G$-invariant.
    Suppose $u\in U$, $g\in G$ and $v\in U^\perp$.
    Then, because $U$ is a subrepresentation, $u':= \rho(g^{-1})u$ is an element of $U$.
    Thus,
    \begin{equation}
        0 = \sprod{u', v} = \sprod{\rho(g)u', \rho(g) v} = \sprod{u, \rho(g) v},
    \end{equation}
    where we used that the representation is orthogonal. Because $u \in U$, $g\in G$ and $v\in U^\perp$ can be chosen arbitrary, this implies that $U^\perp$ is $G$-invariant.
    In particular, $U^\perp$ is subrepresentation of $W$.

    We can now construct the decomposition through a  quick induction argument over the dimension of $W$.
    If $\dim W = 1$, $W$ is already irreducible.
    So suppose $\dim W = n+1\geq 2$ and that all subrepresentations of dimension less than or equal to $n$ factor into orthogonal subrepresentations.
    If $W$ is irreducible, then we are done, so assume that $W$ is reducible, i.e. there exists a non trivial subrepresentation $U$ of $W$.
    Then $U^\perp$ is also subrepresentation of $W$.
    Both have smaller dimensions than $W$ and thus factor into orthogonal products
    \begin{equation}
        U = \bigoplus_i U_i, \ U^\perp = \bigoplus_j V_j
    \end{equation}
    of some irreps $U_i$ and $V_j$ which are orthogonal to each other because $U$ and $V$ are orthogonal to each other.
    So $W = \bigoplus U_i \oplus \bigoplus_j V_j$ is a orthogonal factorisation into irreps.
\end{proof}
Applying Maschke to our space $W$ shows that we can decompose it into irreps
\begin{align} \label{eq:irreps}
    W = W_1 \oplus \dots \oplus W_r
\end{align}
for some $r\in \N$. We can naturally bundle the irreps into groups, based on whether they are isomorphic or not. If we call the non-isomorphic irreps $W_\mu$, and let the multiplicity of $W_\mu$ in \eqref{eq:irreps} be denoted $n_\mu$, we can write
\begin{align*}\label{eq:W_decomp}
    W = \bigoplus_\mu \underbrace{W_\mu \oplus \dots \oplus W_\mu}_{n_\mu \text{ times}} \sim \bigoplus_\mu W_\mu \otimes \R^{n_\mu} =: \bigoplus_\mu W^\mu.
\end{align*}
We can, and will, think of $W^\mu$ as a 'block' of isomorphic irreps.

The definition of $\mathsf H$ involves $\End_G(W)$, the set of equivariant homomorphisms on $W$. A very important lemma for characterizing them is Schur's Lemma.

\begin{theorem}(Schur's Lemma) \parencite[Lemma 1.7]{fulton2013representation}
    Let $W$ and $V$ be irreducible representations of $G$, and $\varphi$ a $G$-equivariant homomorphism between $W$ and $V$.
    Then $\varphi$ is either an isomorphism or the zero map.
\end{theorem} 

Schur's Lemma implies that $\End_G(W)$ consists of 'block-diagonal' operators, in the following sense. First, let us note that we can decompose any operator $T\in \End(W)$ into operators $T^{\lambda,\mu}$ between the blocks $W^\lambda \to W^\mu$. Those operators can further be decomposed into $n_\mu \cdot n_\mu$ blocks $T^{\lambda,\mu}_{i,j}$ between the individual copies of $W_\mu$ and $W_\mu$. In formulas,
\begin{align}
    \End(W) = \Hom(W, W) & \simeq \bigoplus_{\mu, \eta}\Hom(W^\mu, W^\eta) \simeq \bigoplus_{\mu, \eta}\Hom(W_\mu, W_\eta) \otimes_\R \Hom(\R^{n_\mu}, \R^{n_\eta}).
\end{align}
Now, it is not hard to convince oneself that an operator $T$ is equivariant if and only if each individual block $T^{\mu,\eta}_{i,j}$ is. Consequently,
\begin{align}
    \End_G(W) \simeq\bigoplus_{\mu, \eta}\Hom_G(W_\mu, W_\eta) \otimes_\R \Hom(\R^{n_\mu}, \R^{n_\eta}).
\end{align}
Here, we formally used that $G$ acts trivially on $\R^{n_\mu }$ -- that is, $G$ only acts in each individual blocks, and do not move them  around or mix them. Now, Schur tells us that $\Hom_G(W_\mu, W_\eta)$ is trivial for $\mu \neq \eta$, so that 
\begin{align}
    \End_G(W)\simeq \bigoplus_\mu \End_G(W_\mu)\otimes M_{n_\mu}(\R),
\end{align}
where $M_{n_\mu}$ denotes the space of real $(n_\mu\times n_\mu)$-matrices.
Let us decipher this last equation: It says that all $T\in \End_G(W)$ has the following 'nested block-diagonal form':
\begin{align*}
    T = \begin{bmatrix}
        T^{\mu_1} &    & \\
         & \ddots &  \\
         &  & T^{\mu_r}
    \end{bmatrix},  \quad \text{ with } T^{\mu} = \begin{bmatrix}
        T_{1,1}^{\mu} & \dots &  T_{1,n_\eta}^\mu \\
        \vdots & \ddots & \vdots \\
         T_{n_\mu,1}^{\mu} & \dots &  T_{n_\mu,n_\eta}^{\lambda}
    \end{bmatrix}, \ \text{ with }T_{k,\ell}^{\mu} \in \End_G(W_{\mu})
\end{align*}

For future reference, let us formulate this as a proposition.
\begin{prop}\label{prop:End_decomp}
    Let $G$ be a compact group acting on $W$.
    Then, 
    \begin{equation*}
       \End_G(W) \simeq \bigoplus_\mu \End_G(W_\mu)\otimes_\R M_{n_\mu}(\R)
    \end{equation*}
   where $W_\mu$ are the irreducible real representations of $G$, and $n_\mu$ their multiplicity in $W$.
\end{prop}

\subsection{Equivariant endomorphisms for an irreducible representation}\label{subsec:EquivEndIrreps}
Proposition \ref{prop:End_decomp} shows that the analysis of equivariant  endomorphisms on $W$ reduces to the analysis of endomorphisms of an irrep. To do the latter is the purpose of this section. To unclutter the notation, let us in this section write $W$ instead of $W_\mu$, always let $W$ be be a finite-dimensional irrep, and write $D$ instead of $\End_G(W)$.

It will provide very useful to view $W$ a \emph{$D$-module}. $D$-modules are essentially vector spaces in which it is clear what it means to multiply a vector not only with scalars, but with elements in a \emph{ring} $D$. More formally, $W$ is an abelian group $(W, +)$ equipped with a 'scalar' multiplication  $\cdot : D \times W \to W$ defined by application of the endomorphisms, i.e. $\varphi \cdot x = \varphi(x)$. This definition fulfills the abstract axioms for a $D$-module: 
\begin{align*}
        \varphi\cdot(x+y) &= \varphi\cdot x + \varphi\cdot y \quad &\text{(distributivity I)}&& (\varphi+\xi)\cdot x &= \varphi\cdot x + \xi\cdot x \quad &\text{(distributivity II)}  \\
        (\varphi\xi)\cdot x &=\varphi\cdot(\xi x) & \text{(associativity)} &&        I\cdot x &= x &\text{(compatibility with 1)}.
\end{align*}
Here, note that $I\in D$ is the identity. To make it clear, we will denote $W$ when viewed as a $D$-module by $DW$. 

Schur's Lemma implies that every element in $D$ is invertible. Hence, $D$ is a \emph{division ring}, a ring in which each non-zero element has a multiplicative inverse. Division rings behave like fields, just without commutativity in the multiplication. $D$ being a division ring implies that $DW$ has a basis - a set of vectors in $W$ which generate every element in $W$ as a $D$-linear combination, and which are linearly independent over $D$. Let us give the quick proof that said basis is finite.

\begin{prop}
    $DW$ has a finite basis.
\end{prop}
\begin{proof}
    Note that $D$ contains $\R$ via right multiplication by $\lambda I$ for some $\lambda\in\R$.
    The restricted scalar multiplication of $D$ to $\R$ corresponds to the normal scalar multiplication of $\R$.
    Therefore, there exists a finite generating set $(x_i)_i$ of $W$ over $\R$, which also is generating over $D$.

    Now choose a minimal generating subset of $(x_i)_i$ over $D$ which we will denote by $(y_i)_i$.
    We claim that this set is a basis over $D$. We only need to prove linear  independence, so let $\lambda_i \in D$ be such that
    \begin{equation}
        0 = \sum_{i=1}^r \lambda_i y_i.
    \end{equation}
    Suppose $\lambda_j\neq 0 $ for some $1\leq j\leq r$, then we have
    \begin{equation}
        y_j = -\sum_{i\neq j}\lambda_j^{-1}\lambda_i y_i,
    \end{equation}
    so $(y_i)_i$ is not a minimal generating set.
    Therefore, there cannot exist a $\lambda_i\neq 0$ and thus $(y_i)_i$ is a $D$-basis of $W$.
\end{proof}

$D$ has even more structure than being a division ring: It is also a finite dimensional $\R$-vector space, as a subspace of $\End(W)$. It is hence a \emph{finite dimensional real division algebra}.
Those finite dimensional division algebras over $\R$ were famously characterized by Frobenius in 1877.
\begin{theorem}(Frobenius Theorem on division algebras over the reals)
    Let $D$ be a finite dimensional division algebra over $\R$, i.e. $D$ is a division ring and a $\R$ vector space.
    Then $D$ is isomorphic to either $\R, \C$ or the \emph{quaternions} $\Hh$.
\end{theorem}

\begin{remark}
    Let us give a refresher on the quaternions $\Hh$. We can think of them as a real vector space with basis $\{1, j, k , \ell\}$ equipped with a multiplication for which $j^2 = k^2 = \ell^2 = -1$ and $jk\ell = -1$. We will refer to $j,k$ and $\ell$ as \emph{imaginary units}.
    Similarly to the complex numbers, $\Hh$ comes with a complex conjugation
    \begin{equation}
        \overline{a + jb + kc + \ell d} = a -jb - kc - \ell d
    \end{equation}
    for $a, b, c, d\in \R$. 
    For a quaternion $z = a + jb + kc + \ell d$, we call $a$ the \emph{real part} of $z$ and $b, c, d$ the imaginary parts corresponding to $i, j$ and $k$ respectively.
    We will denote these by $\Re z, \Im^jz, \Im^k z$ and $\Im^\ell z$ and note that a quaternion is completely characterized by its real and imaginary parts.
\end{remark}

We will speak $W$ being of \emph{real}, \emph{complex} or \emph{quaternionic} type respectively if $D$ is isomorphic to $\R, \C$ or $\Hh$. The meaning of isomorphy here is that there exists a $\R$-linear map $\phi$ from $\F$, where $\F$ is either $\R$, $\C$ or $\Hh$, into $D$ that respects the multiplication operation, i.e.
\begin{align} \label{eq:homomorphism}
   \phi(\lambda z) = \lambda \phi(z), \quad \phi(z+w)=\phi(z)+\phi(w), \quad  \phi(zw)= \phi(z)\phi(w), \quad \quad  \lambda \in \R, z,w\in \F.
\end{align}
$\phi$ gives us a way to, in the same sense that $\R$ is embedded in $D$ through $\lambda \leftrightarrow \lambda \cdot \id$, embed the entirety of $\F$ in $D$  
through  the identification of $z\in \F$ with its 'copy' $Z=\phi(z)\in D$. 
The embedding is sound, in the way that multiplying copies as endomorphisms is the same as multiplying the numbers. The identification even respects the conjugation operation, in the following sense. 

\begin{prop}\label{prop:compl_correspond_orthadj}
    The isomorphism $\phi:\F\to D$ fulfils
    \begin{align}
        \phi(\overline{z})= \phi(z)^*,
    \end{align}
    where $*$ denotes the adjoint.
\end{prop}
\begin{proof}
    For $z=\lambda\in \R$, the claim is trivial, since $\phi(\lambda)=\phi(\lambda \cdot 1) = \lambda \cdot \phi(1) = \lambda \cdot I$, and consequently
    \begin{equation}
        \phi(\overline{\lambda})=\overline{\lambda}\cdot I  \stackrel{\lambda\in \R}{=} \lambda \cdot I = (\lambda \cdot I)^*= \phi(\lambda)^*.
    \end{equation}
    Now let $j \in \F$ be an imaginary unit, i.e. $j^2=-1$, and denote $J=\phi(j)$. $J^*J\in \End_G(W)$ is a positive definite symmetric endomorphism, and hence only has real, positive eigenvalues. Eigenspaces of $J^*J$ are $G$-invariant, since $J^*J$ is $G$ equivariant.
    Since $W$ is an irrep, there can therefore only be one eigenvalue, and $J^*J = \mu I$ for a $\mu\geq 0$. Multiplying this equation from the left with $J^*$ and the right with $J$ yields
    \begin{align}
        \mu^2 \cdot I = \mu \cdot  J^*J = J^*J^*JJ = (J^2)^*J^2 = (-I)^*(-I)=I.
    \end{align}
    Hence, since $\mu\geq 0$, $\mu=1$. Consequently, $J^*J=I$, and since $J\cdot J= \phi(j^2)=\phi(-1)=-I$, we obtain $J^*=-J$ (and also $K^*=-K$, $L^*=-L$).

    We can now argue for a general $z=a+bj + ck + d\ell \in \Hh$, $a,b,c,d\in \R$, that
    \begin{align}
        \phi(a+bj+ck+d\ell)^* &= (a\phi(1))^* + (b\phi(j))^* + (c\phi(k))^* + (d\phi(\ell))^* \\
        &= (a I)^* + (bJ)^* + (cK)^* + (dL)^* =  a I  -bJ - cK - dL \\
        &= \phi(a-bj-ck-d\ell)
    \end{align}
    i.e. the claim (the exact same argument applies in $\C$, of course). 
\end{proof}
This means that we can truly identify elements $z\in \F$ with their copies $Z\in D$. We will always note this correspondence in this manner of capitalization in the following.

\begin{remark}
The Frobenius theorem in particular puts restrictions on which dimensions $W$ can have, since it is not hard to see that
\begin{equation}
    \dim_\R W = \dim_\R D \cdot \dim_D DW =  \begin{cases}
        1 \dim_D DW & D\simeq \R \\
        2 \dim_D DW & D\simeq \C \\
        4 \dim_D DW & D\simeq \Hh.
    \end{cases}
\end{equation}
\end{remark}

These examples might help to understand the structure a little bit more.
\begin{example}\label{ex:G_and_D}
        \begin{enumerate}
            \item Consider the standard action of the orthogonal group $\O(n)$ on $W = \mathbb{R}^n$, $\rho(Q)w=Qw$. Since there for any $v\neq v'$ of equal norm exists a rotation $Q$ with $Qv=v'$, $W$ is an irrep of this action. In this case,
            \begin{equation} \label{eq:realrep}
                D = \End_G(W) = \{\lambda \cdot I  :\lambda\in\R\} \simeq \R.
            \end{equation}
            This is hence an example of a real representation, i.e. when $\F=\R$.
            
            There are many ways to convince oneself of \eqref{eq:realrep}. The most mundane way is probably as follows: If $D\ni Q\neq \lambda \cdot I$, there exists a vector $v$ with $Qv\notin \mathrm{span}(v)$. We can write $Qv=\alpha v+\beta w$ for some $\alpha,\beta\in \R$ and $w\perp v$. Now let $U \in \O(W)$ satisfy $Uv= v$, $U w =-w$. If $Q\in D$, $Q$ commutes with $U$, which yields $Qv =QUv=UQv= \alpha v - \beta w$. Equating these yields $\beta=0$. Hence, $Qv=\alpha v$, which is a contradiction.
            
            \item Consider the action of the rotation group 
            \begin{equation}
                \SO(2) = \left\{\begin{bmatrix}
                    \cos(\theta) & - \sin(\theta) \\ \sin(\theta) & \cos(\theta)
                \end{bmatrix} \, \bigg\vert \, \theta \in [0,2\pi]\right\}
            \end{equation} 
            on the plane $\R^2$. $\R^2$ is again an irrep for this action.  Which matrices commute with all these matrices? Since rotations in the plane commute, clearly all scaled rotations do.            
            On the other hand, by checking the commutation relation for a general matrix with the matrix corresponding to $\theta=\pi/2$, one sees that it must be of the form
            \begin{equation}
                \begin{bmatrix}
                    \lambda & -\mu \\ \mu & \lambda 
                \end{bmatrix}     = \sqrt{\lambda^2+\mu^2} \begin{bmatrix}
                    \frac{\lambda}{\sqrt{\lambda^2+\mu^2}} & \frac{-\mu}{\sqrt{\lambda^2+\mu^2}} \\ \frac{\mu}{\sqrt{\lambda^2+\mu^2}} & \frac{\lambda}{\sqrt{\lambda^2+\mu^2}}
                \end{bmatrix} =\sqrt{\lambda^2+\mu^2}\begin{bmatrix}
                   \cos(\omega)& - \sin(\omega) \\ \sin(\omega) & \cos(\omega)
                \end{bmatrix} ,
            \end{equation}
            i.e. a scaled rotation.
            
            Now note that the scaled rotations are isomorphic to $\C$! The complex number $z=re^{i\theta}$ corresponds to the $r$-multiple of a rotation with $\theta$ degrees. This is hence an example of a complex representation. In particular, the imaginary unit $j$ can identified with
            \begin{align}
                J =\begin{bmatrix} \label{eq:imunit}
                    0 & - 1 \\ 1 & 0
                \end{bmatrix}.
            \end{align}
            We see that this indeed is an orthogonal operator with $J^2=-I.$

            \item The cases of quaternionic type is necessarily a bit less geometrically intuitive, since they need to take place in vector spaces of at least dimension $4$. The simplest example is therefore given by an action on $W=\R^4$, namely of the \emph{special unitary group} $\mathrm{SU}(2)$, i.e.
            \begin{equation}
                \mathrm{SU}(2) = \{U \in U(2) \, \vert \, \det(U)=1 \} = \left\{\begin{bmatrix}
                    a & -\overline{b} \\  b &  \overline{a}
                \end{bmatrix} \ \bigg| \ a,b \in \C, |a|^2+|b|^2=1 \right\}.
            \end{equation}
            $\mathrm{SU}(2)$ acts on vectors in $W$ as if they were vectors in $\C^2$. Or in  more 'real' terms
            \begin{equation}
                \rho\left(\begin{bmatrix}
                    a & -\overline{b} \\  b &  \overline{a}
                \end{bmatrix}\right) = \begin{bmatrix} A & - B^* \\
                B & A^*
                    \end{bmatrix},
            \end{equation}
            where the capitalization refers to the identification of a complex number as an operator on $\R^2$, as in the previous example. $W$ is an irrep under this action, since applying the operators with $A,B =(I,0)$, $(J,0)$, $(0,-I)$,  and $(0,J)$ ($J$ as in \eqref{eq:imunit}) to the first unit vector $e_1$  yields $e_1,e_2,e_3$ and $e_4$, respectively. For instance,
            \begin{align*}
                \begin{bmatrix}
                    J & 0 \\
                    0 & J^* 
                \end{bmatrix} \cdot e_1 = \begin{bmatrix}
                    0 & -1 & 0 & 0 \\
                    1 & 0 & 0 & 0 \\
                    0 & 0 & 0 & 1\\
                    0 & 0 & -1 & 0 \end{bmatrix} \begin{bmatrix}
                        1 \\ 0 \\ 0 \\ 0 
                    \end{bmatrix}= \begin{bmatrix}
                    0 \\ 1 \\ 0 \\ 0 
                \end{bmatrix}.
            \end{align*}

            Now what is $\End_G(W)$? By checking when a general block matrix in $\R^{4\times 4}$ commutes with $A=I, J$ and $B=1,J$ respectively,  we see that the matrices in $\End_G(W)$ has the form
            \begin{equation}
                \begin{bmatrix}
                    M & -N \\
                    N & M
                \end{bmatrix},\text{ where } MJ=JM \text{ and } NJ=-JN
            \end{equation}
            The commutation relations of $M$ and $N$ give that $M$ must have the form
            \begin{equation}
                M = \begin{bmatrix}
                    \lambda_1  & - \lambda_2 \\ \lambda_2 & \lambda_1
                \end{bmatrix}, \quad N = \begin{bmatrix}
                    \mu_1 & \mu_2 \\ \mu_2 & -\mu_1
                \end{bmatrix} , \quad \lambda_1, \lambda_2, \mu_1, \mu_2 \in\R.
            \end{equation}
            Hence, $D$ is four-dimensional, which already shows that we are in the quaternionic case. To  make it more explicit, we see that $D$ is spanned by $I$ and 
            \begin{align*}
                \widehat{J}= \begin{bmatrix}
                    0 & -1 & 0 & 0 \\
                    1 & 0 & 0 & 0 \\
                    0 & 0 & 0 & -1 \\
                    0 & 0 & 1 & 0
                \end{bmatrix}, \ \widehat{K} =  \begin{bmatrix}
                    0 & 0 & -1 & 0 \\
                    0 & 0 & 0 & 1 \\
                    1 & 0 & 0 & 0 \\
                    0 & -1 & 0 & 0
                \end{bmatrix}, \ \widehat{L} = \begin{bmatrix}
                    0 & 0 & 0 & -1 \\
                    0 & 0 & -1 & 0 \\
                    0 & 1 & 0 & 0 \\
                    1 & 0 & 0 & 0
                \end{bmatrix}, 
            \end{align*}
            which fulfil $\widehat{J}^2 = \widehat{K}^2 = \widehat{L}^2= \widehat{J}\widehat{K}\widehat{L}=-I$.
             
        \end{enumerate}
\end{example}

\subsection{Orthogonal stabilizers of group equivariant endomorphisms}\label{subsec:OrthogonalStabilizers}
To classify $\mathsf{H}$, we must study the orthogonal maps which commute with symmetric $G$ equivariant maps.
To do so, it is due to Proposition \ref{prop:End_decomp} helpful to characterize the orthogonal maps which commute with $G$-equivariant endomorphisms of irreducible components (cf. subsection \ref{subsec:EquivEndIrreps}). This is the purpose of this section. The broad idea will be to, using that $W$ is a $\End_G$-module, equip $W$ with either a real, complex, or quaternionic inner product, and then identify the sought for commuting maps as the pendants of the orthogonals for the respective cases: \emph{orthogonal}, \emph{unitary} or \emph{symplectic} maps.

\paragraph{Equipping $W$ with an inner product}
Due to the Frobenius theorem, $D=\End_G(W)$ is isomorphic to some $\F \in \{\R,\C, \Hh\}$. As a subalgebra of $\End(W)$ which is isomorphic to $\F$, where furthermore taking the orthogonal adjoint corresponds to conjugation in $\F$ (that is, $D$ is isomorpic to $\F$ as a \emph{$*$-algebra}), it is a so-called  \emph{$\F$-structure on $W$}. We will in the following characterize the endomorphisms that commute with the $\F$-structure. We adopt the identification between $z\in \F$ and $Z\in D$ described above.

As an $\F$-vector space, $\F W$ is isomorphic to $\F^n$ for some $n\in\N$.
Because $\F$ is either $\R$, $\C$ or $\Hh$, the vector space $\F W$ can be equipped with an inner product $h: \F W\times \F W \to \F$. Inner products are defined by
\begin{align*}
    h(\lambda x +y, z) & = \lambda h(x, z) + h(y, z) &&\text{(Linearity)} \\
    \overline{h(x, y)} &= h(y, x) &&\text{(Conjugate symmetry)} \\
    h(x, x) &> 0 \text{ for } x\neq 0&&\text{(Positive definiteness)}
\end{align*}
for all $x, y, z\in \F W$ and $\lambda\in \F$.

In our case, there exists a \emph{canonical} inner product on $\F W$ which \emph{extends} the inner product $\sprod{., .}$ on $W$ uniquely as an $\F$-bilinear map for which $\Re h(x, y) = \sprod{x, y}$. For $F=\R$, this is trivial.
For $\F =  \C$ or $\F= \Hh$, such an $h$ must  any imaginary unit $j$ fulfil
\begin{equation}\label{eq:inner_prod_def}
    \sprod{Jx, y}  = \Re h(Jx, y) = \Re jh(x, y) = - \Im^j h(x, y).
\end{equation}
The only way this is possible is if the imaginary part with regards to $j$ is equal to $ -\sprod{Jx, y}$.
Similarly, the imaginary parts of $k$ and $\ell$ must correspond to $-\sprod{Kx, y}$ and $-\sprod{Lx, y}$ respectively.
Therefore, we have uniquely characterized $h$ by its imaginary parts and get
\begin{equation}
    h(x, y) = \begin{cases}
        \sprod{x, y} & D\simeq \R \\
        \sprod{x, y} - j\sprod{Jx, y} & D\simeq \C \\
        \sprod{x, y} - j\sprod{Jx, y} - k\sprod{Kx, y} - \ell\sprod{Lx, y} & D\simeq \Hh
    \end{cases}. \label{eq:definprod}
\end{equation}
Note that due to linearity in $x$, this inner product is independent on the choice $j, k$ and $\ell$ as long $J, K$ and $L$ get changed accordingly.

We still need to argue that \eqref{eq:definprod} actually yields an inner product, which we do now.
\begin{prop}
    $h$ is an inner product.
\end{prop}
\begin{proof}
    We will carry out the proof only for the quaternionic case, from which the complex and real cases will follow.
    
    As a sum of additive maps, $h(\cdot, y)$ is additive as well.
    So, we only need to prove the compatibility with scalar multiplication for linearity.
    It suffices to show this for $\lambda \in \{1,j,k,\ell\}$. The case $\lambda=1$ is trivial. For $\lambda = j$, we get
    \begin{align}
        h(j x, y) &= \sprod{J x, y} - j \sprod{Jj x, y} - k \sprod{Kj x, y} - \ell\sprod{Lj x, y} \\
        &= \sprod{J x, y} - j \sprod{J^2 x, y} - k \sprod{KJ x, y} - \ell\sprod{LJ x, y} \\
        &=  \sprod{Jx, y} + j\sprod{x ,y}  + k\sprod{Lx, y} + \ell\sprod{Kx, y}\\
        &= j\left(- j\sprod{Jx, y} +\sprod{x, y}  - \ell\sprod{Lx, y} - k\sprod{Kx, y} \right) \\
        &= j h(x, y)
    \end{align}
    for all $x, y\in \F W$ as required.

    Now we want to prove the conjugate symmetry.
    For $x, y\in \F W$, we see that
    \begin{align}
        \overline{h(x,y)} &= \sprod{x, y} + j\sprod{Jx, y} + k\sprod{Kx, y} + \ell\sprod{Lx, y} \\
        &= \sprod{x, y} + j\sprod{x, J^*y} + k\sprod{x, K^*y} + \ell\sprod{x, L^*y} \\        
        &= \sprod{x, y} - j\sprod{x, Jy} - k\sprod{x, Ky} - \ell\sprod{x, Ly} \\
        &= \sprod{y, x} - j\sprod{Jy, x} - k\sprod{Ky, x} - \ell\sprod{Ly, x} \\
        &= h(y, x),
    \end{align}
    where we used that $J^* = -J$ and that $\sprod{., .}$ is symmetric. The arguments for $k$ and $\ell$ are analogous.

    Finally,
    \begin{equation}
        \sprod{x, Jx} = \sprod{Jx, x} = \sprod{x, J^*x} = - \sprod{x, Jx}
    \end{equation}
    and thus $\sprod{x, Jx} = \sprod{x, Kx} = \sprod{x, Lx} = 0$.
    Consequently, $h(x,x) = \sprod{x, x}$ which is a norm, and thus positive for $x\neq 0$.
\end{proof}

The above construction gives us the following:
\begin{prop}
There is a bijection
    \begin{align}
        \left\{ \text{real Hilbert spaces with a $\F$-structure}\right\} &\leftrightarrow \left\{\text{$\F$ vector spaces with an inner product}\right\} \\
        (W, \sprod{. ,.}) & \mapsto (\F W, h(., .))\\
    \end{align}
    where $\F = \R, \C$ or $\Hh$.
\end{prop}

It turns out that the orthogonal stabilizers of $\End_G(W)$ are exactly those that lift to $\F$-linear maps on $\F W$ which keep the inner product $h$ fixed, i.e. 
\begin{equation}
    h(Qx, Qy) = h(x, y)
\end{equation}
for all $x, y\in \F W$. Such maps are known as orthogonal (if $\F\simeq \R$), unitary (if $\F\simeq \C$) or symplectic (if $\F\simeq \Hh$).
We will denote the orthogonal group $\O(\R W)$, the unitary group $\U(\C W)$ and the compact symplectic group $\Sp(\Hh W)$ respectively.

\begin{remark}
One should not confuse the compact symplectic group $\Sp(\Hh W)\sim \Sp(n)$ with the symplectic group $\Sp(2n,\C)$. The latter is usually defined as the set of endomorphisms on $\C^{2n}$ that keeps a bilinear antisymmetric form fixed. $\Sp(n)$ is the intersection of $\Sp(2n,\C)$ with the unitary group $\U(2n)$. In fact, because of the way $\Sp(\Hh W)$ enters our analysis, it might be more illustrative to use the less common term \emph{quaternionic unitary group}. 

For simplicity, we will refer to elements of $\Sp(\Hh W)$ as \emph{symplectic}.

\end{remark}

\begin{lemma}\label{lem:Dlinmaps_exactly_orthunisym}
    The maps which commute with $D$ and are orthogonal with regards to $(W, \sprod{.,.})$ are exactly the orthogonal/unitary/symplectic maps of $(\F W, h(., .))$.
\end{lemma}
\begin{proof}
    Again, we will only present the proof for the quaternionic case.
    In this proof, we will use the same notation for elements $x, y$ in $W$ and $\F W$, as $W$ and $\F W$ are equal as sets.
    For the same reason, we use the same notation for an $\F$-linear $Q\in \End_\R(W)$ and the corresponding $Q\in \End_\F(\F W)$. 

    Let $Q$ be an orthogonal map in $(W, \sprod{. ,.})$ which commutes with $D$. 
    Then it can be viewed as a linear map on $\F W$ with the identification $D\simeq\F$.
    We then get
    \begin{align}
        h(Qx, Qy) &= \sprod{Qx, Qy} - j\sprod{JQx, Qy} - k\sprod{KQx, Qy} - \ell\sprod{LQx, Qy} \\
        &=\sprod{Qx, Qy} - j\sprod{QJx, Qy} - k\sprod{QKx, Qy} - \ell\sprod{QLx, Qy} \\
        &=\sprod{x, y} - j\sprod{Jx, y} - k\sprod{Kx, y} - \ell\sprod{Lx, y} \\  &= h(x, y)
    \end{align}
    for all $x, y\in \F W$. That is, $Q$ preserves $h$, as required.

    So now let $Q$ be symplectic map in $\F W$.
  
    $Q$ is as a map $W\to W$ $\R$-linear, which furthermore, 
    with the identification of $D\simeq \F$,  commutes with $D$.
    To show that it is orthogonal, we simply need to note
    \begin{equation}
        \sprod{Qx, Qy} = \Re h(Qx, Qy) = \Re h(x, y) = \sprod{x, y}
    \end{equation}
    for all $x, y\in W$ as required.
\end{proof}

\begin{example}
    Let us revisit Example \ref{ex:G_and_D} and apply our results.

    We have seen that scaled rotations only commute with scaled rotations.
    Intersecting these with the orthogonal group, we get back $SO(2)$.
    This is not surprising as $S^1 \simeq U(1)\simeq SO(2)$, where the isomorphism is given through the identification $\C\simeq \R^2 \simeq \R \oplus j \R$ as vector spaces.
    The isomorphism is given by
    \begin{equation}
        \exp(i\theta) \mapsto \begin{bmatrix}
            \Re \exp(i\theta) & -\Im \exp(i\theta) \\
            \Im \exp(i\theta) & \Re \exp(i\theta)
        \end{bmatrix} = \begin{bmatrix}
            \cos \theta & - \sin \theta \\
            \sin \theta & \cos \theta
        \end{bmatrix}.
    \end{equation}

    Similarly, in the described quaternionic case, we have $S^3 \simeq Sp(1) \simeq SU(2)\subseteq O(4)$ where the isomorphism is given through the identification $\Hh\simeq \C^2 \simeq \C\oplus k\C \simeq \R \oplus j\R \oplus k\R \oplus \ell\R$ with $\ell = -jk$.
    Explicitly, we have
    \begin{equation}
        (a + jb + kc + ld) = z \mapsto \begin{bmatrix}
            \Re z & - \Im^j z & - \Im^k z & -\Im^\ell z \\
            \Im^j z & \Re z  & \Im^\ell z & - \Im^k z\\
            \Im^k z & -\Im^\ell z &  \Re z & - \Im^j z\\
            \Im^\ell z &  \Im^k z & \Im^j z & \Re z
        \end{bmatrix} = \begin{bmatrix}
            a & -b & -c & -d \\
            b & a & d & -c \\
            c & -d & a & -b \\
            d & c & b & a 
        \end{bmatrix}
    \end{equation}
    where $1 =|z|^2 = a^2 + b^2 + c^2 + d^2$.
    Note that the matrix on the right is really an element of $\O(4)$.

    These isomorphisms naturally extend to higher dimensions.
    Let $W = \R^{2n}$, which we identify with $\C^n \simeq \R^n \oplus j \R^n$.
    In this identification, multiplying by $j$ has the following representing matrix
    \begin{equation}
        J = \begin{bmatrix}
            0 & - I_n \\
            I_n & 0
        \end{bmatrix}.
    \end{equation}
    Then the orthogonal matrices which commute with $D = \mathrm{span}_\R(I, J)\simeq \C$ are given by the elements of $\U(n)$ represented as elements in $\O(2n)$ via
    \begin{equation}
        Q \mapsto \begin{bmatrix}
            \Re Q & - \Im Q \\
            \Im Q & \Re Q
        \end{bmatrix} =:\widehat{Q}
    \end{equation}
    and note that $Q^*Q = I_{n}$ with regards to the Hermitian adjoint $Q^* = \overline{Q^\top }$ corresponds to 
    \begin{equation}
        \Re Q^\top \Re Q + \Im Q ^\top \Im Q = I_n, \ \Re Q^\top \Im Q - \Im Q^\top \Re Q  = 0
    \end{equation}
    which is (not) coincidentally the requirement for $\widehat{Q}$ being an element of $\O(2n)$. 

    We will write down the quaternionic case for the sake of completeness but leave out concrete calculations.
    Let $W = \R^{4n}$ which we will identify with $\Hh^n \simeq \R^n \oplus j \R^n \oplus k\R^n \oplus l \R^n$.
    With this identification, we have
    \begin{equation}
        J = \begin{bmatrix}
            0 & -I_n & 0 & 0 \\
            I_n & 0 & 0 & 0 \\
            0 & 0 & 0 & -I_n \\
            0 & 0 & I_n & 0
        \end{bmatrix}, \ K = \begin{bmatrix}
            0 & 0 & -I_n & 0 \\
            0 & 0 & 0 &  I_n \\
            I_n & 0 & 0 & 0 \\
            0 & -I_n & 0 & 0
        \end{bmatrix}
    \end{equation}
    and $L = -JK$.
    Then, the elements of $\O(4n)$ commuting with $D = \R\sprod{I, J, K, L} \simeq\Hh$ are given by $\Sp(n)$ through the following embedding
    \begin{equation}
        Q\mapsto \begin{bmatrix}
            \Re Q & - \Im^j Q & - \Im^k Q & -\Im^\ell Q \\
            \Im^j Q & \Re Q  & \Im^\ell Q & - \Im^k Q\\
            -\Im^k Q & -\Im^\ell Q &  \Re Q & - \Im^j Q\\
            \Im^\ell Q & - \Im^k Q & \Im^j Q & \Re Q
        \end{bmatrix} =: \widehat{Q} 
    \end{equation}
    where $\Im^j$, $\Im^k$, $\Im^\ell$ are the respective imaginary parts.
    Then $Q$ being symplectic corresponds exactly to $\widehat{Q}$ being orthogonal.
\end{example}

\subsection{Finalizing the argument}
\label{app:putting_together}
Now let us finalize the analysis.
We no longer consider an arbitrary Hilbert space $W$ but go back to our Hilbert space of interest, $\calH$. 
Remember the definition of $\mathsf{H}$, cf. Equation \ref{eq:Hs}
\begin{equation}
        \mathsf H = \{B\in \mathrm{GL}(\calH)  \, \vert \, 
        BSB^* =S \text{, } S\in \Sym^2_G(\calH) \\
        LB^* = L \text{ and } L\in \Hom_G(\calH,\calY) \}
\end{equation}
Note that because $I\in \Sym^2_G(\calH)$, all elements in $\mathsf{H}$ must satisfy $BB^*=BIB^*=I$, and are hence orthogonal.
Consequently, 
we can write
\begin{equation}
    \mathsf{H} = \mathsf{H}^1 \cap \mathsf{H}^2
\end{equation}
where
\begin{align}
    \mathsf{H}^1 &= \{Q\in \O(\calH) : Q^*SQ = S, \ S\in\Sym^2_G(\calH)\} \\
    \mathsf{H}^2 &= \{Q\in \O(\calH): LQ = L, \ L\in \Hom_G(\calH, \calY) \}.
\end{align}
We will characterize first $\mathsf{H}^1$ and then $\mathsf{H}^2$.

\subsubsection{Characterizing $\mathsf{H}^1$}

$\mathsf{H}^1$ is the group commuting with all \emph{symmetric} equivariant endomorphisms. Let us begin by describing those. We have shown in Proposition \ref{prop:End_decomp} that
\begin{equation}
    \End_G(\calH) \simeq \bigoplus_\mu\End_G(\calH^\mu) \simeq \bigoplus_\mu\End_G(\calH_\mu) \otimes_\R M_{n_\mu}(\R). 
\end{equation}
We can extend this chain of canonical isomorphisms even further as
\begin{equation}
    \End_G(\calH_\mu) \otimes_\R M_{n_\mu}(\R) \simeq M_{n_\mu}(D_\mu)
\end{equation}
where we denote $D_\mu = \End_G(\calH_\mu)$. To see this, denote with $E_{ij} \in M_{n_\mu}(\R)$ the matrices which are zero everywhere but $1$ at the $(i,j)$ entry.
These matrices form a basis of $M_{n_\mu}(\R)$.
We can write elements in $T = \End_G(\calH_\mu) \otimes_\R M_{n_\mu}(\R)$ as an unique linear combination in $E_{ij}$ of the following form
\begin{equation}
    T = \sum_{i,j} T^{ij} \otimes E_{ij}
\end{equation}
with $T^{ij}\in \End_G(\calH_\mu)=D_\mu$.
This immediately gives us the isomorphism: a $T$ as in \eqref{eq:Teq} is sent to the matrix  in $M_{n_\mu}(D_\mu)$  with $(i,j)$:th entry $T^{ij}$.
As a result, every element in $\End_G(\calH^\mu)$ can be viewed as a block matrix where each block corresponds to an element of $D_\mu$, i.e.
\begin{equation}
    T = \begin{bmatrix}
        T_{11} & \ldots & T_{1n_\mu} \\
        \vdots & \ddots & \vdots \\
        T_{n_\mu 1} & \ldots & T_{n_\mu n_\mu}
    \end{bmatrix}
\end{equation}
with $T_{ij} \in D_\mu$.

Let us now present how the orthogonal adjoint factors through all these isomorphisms.
On $\End_G(\calH_\mu) \otensor_\R M_{n_\mu}(\R)$, the induced orthogonal adjoint is given by
\begin{equation}
    T^* = \sum_{i,j} (T^{ij} \otensor E_{ij})^* = \sum_{ij} \left(T^{ij}\right)^* \otensor E_{ji}
\end{equation}
which can be easily verified from the definition of the isomorphism by restricting $\sprod{., .}$ to the irreducible components.
Then, the orthogonal adjoint as an element in $M_{n_\mu}(D_\mu)$ is given by
\begin{equation}
    T^* = \begin{bmatrix}
        T_{11}^* & \ldots & T_{n_\mu1}^* \\
        \vdots & \ddots & \vdots \\
        T_{1 n_\mu}^* & \ldots & T_{n_\mu n_\mu}^*
    \end{bmatrix}
\end{equation}
with $T_{ij} \in D_\mu$.
By Proposition \ref{prop:compl_correspond_orthadj}, the adjoint operator on $D_\mu$ corresponds to complex conjugation on $\F_\mu$.
As a result, an element $A\in M_{n_\mu} (\F_\mu)\simeq\End_G(\calH^\mu)$ is symmetric if and only if $A = A^H = \overline{A^\top }$.
We will denote the set of those matrices by $\Herm_{n_\mu}(\F_\mu)$, the set of Hermitian matrices.

We therefore get the following result:
\begin{prop} \label{prop:SymAreHerm}
    We have the following canonical isomorphism
    \begin{equation}
        \Sym^2_G(\calH^\mu) \simeq \Herm_{n_\mu}(\F_\mu).
    \end{equation}

    In particular, if $n_\mu = 1$, then $\Sym^2_G(\calH^\mu) \simeq \R$.
\end{prop}
\begin{cor}\label{cor:Sym_decomp}
    We have the following canonical isomorphism
    \begin{equation}
        \Sym^2_G(\calH) \simeq \bigtimes_{\mu} \Sym^2_G(\calH^\mu) \simeq \bigtimes_{\mu}  \Herm_{n_\mu}(\F_\mu).
    \end{equation}
\end{cor}
\begin{proof}
    The first isomorphism is a direct result of Proposition \ref{prop:End_decomp}, where we use $\Sym^2_G(\calH) = \End_G(\calH) \cap \Sym^2(\calH)$ and $\Sym^2(\calH) \cap \End(\calH^\mu) = \Sym^2(\calH^\mu)$ with the obvious identification of $\End(\calH^\mu) \subseteq \End(\calH)$. The second one is Proposition \eqref{prop:SymAreHerm}.
\end{proof}

Now let us find which orthogonal matrices $Q$ in $\O(\calH)$ that fix those symmetric endomorphisms by conjugation, i.e.
\begin{equation}
    Q A Q^* = A
\end{equation}
for all $A\in \Sym_G^2(\calH)$. Let $\calH = \bigoplus_i \calH_i$ be the actual orthogonal decomposition of $\calH$ with respect to $G$ as in Equation \ref{eq:irreps}.
Note that we consider here the concrete irreducible components in $\calH$ and not their isomorphic identifications with multiplicities.
Then, we have the following result:
\begin{lemma}
    $\mathsf{H}^1$ is a subgroup of $\bigtimes_i \O(\calH_i) \subseteq \O(\calH)$.
\end{lemma}
\begin{proof}
    We only need to show that $Q\calH_i = \calH_i$ for each $i$, since any unitary wit this property is in $\times_i \O(\calH_i)$. Note that $\Sym^2_G(\calH)$ contains the endomorphisms $I_{\calH_i}\in\End(\calH_i)\subseteq \End(\calH)$ that act as the identity on $\calH_i$, and as the zero map elsewhere.
    Any $Q\in \mathsf{H}^1$ fixes this map by conjugation, meaning that 
    \begin{equation}\label{eq:QI_HQ = I}
        QI_{\calH_i} Q^* = I_{\calH_i},
    \end{equation}
    i.e., that $Q$ restricted to $\calH_i$ acts orthogonally. Now, since $Q$ and $Q^*$ are invertible
    \begin{equation}
        Q(\calH_i) = QI_{\calH_i} (\calH) = QI_{\calH_i} Q^*(\calH) = I_{\calH_i}(\calH) = \calH_i,
    \end{equation}
    which was to be shown.
\end{proof}

Consider again the decomposition in Equation \ref{eq:W_decomp}
\begin{equation}
    \calH \simeq \bigoplus_\mu \calH^\mu \simeq \bigoplus \calH_\mu\otimes \R^{n_\mu},
\end{equation}
where we identify irreducible components $\calH_i$ with their isomorphic representations $\calH_\mu$.
The above corollary implies that $\mathsf{H}^1$ acts diagonally on $\calH^\mu$, i.e. via orthogonal matrices $(Q_i)_{ii} = Q \in \O(\calH_\mu)^{\times n_\mu} \subseteq \O(\calH^\mu)\subseteq \O(\calH)$.

Now, we would like to use Corollary \ref{cor:Sym_decomp} to restrict the elements of $\mathsf{H}^1$ even further.  
In the case where $n_\mu = 1$, we have $\Sym^2_G(\calH^\mu) \simeq \R$, and thus $\O(\calH^\mu) = \O(\calH_\mu)$ fixes this space by conjugation.

So  assume $n_\mu \geq 2$.
Then $\Sym_G^2(\calH^\mu)\simeq \Herm_{n_\mu}(\F_\mu)$ is spanned by the symmetric maps corresponding to
\begin{equation}
    S =\lambda E_{ij} + \overline{\lambda} E_{ji}
\end{equation}
for any possible $(i,j )$ and $\lambda\in \F_\mu$.
Thus, an element $(Q_i)_{ii} = Q \in \O(\calH_\mu)^{\times n_\mu}$ fixes the symmetric matrices by conjugation if and only if it fixes each such matrix. 
As all blocks in $S$ are zero except for the $ij$ and $ji$ ones,  we can restrict our analysis to these components.
Then, the equation can be written in block structure as
\begin{equation}
    \begin{bmatrix}
        Q_{ii} & Q_{ij} \\
        Q_{ji} & Q_{jj}
    \end{bmatrix}\begin{bmatrix}
        0 & \lambda \\
        \overline{\lambda} & 0
    \end{bmatrix}\begin{bmatrix}
        Q_{ii}^* & Q_{ji}^* \\
        Q_{ij}^* & Q_{jj}^*
    \end{bmatrix} = \begin{bmatrix}
        0 & Q_{ii} \lambda Q_{jj}^* \\
        Q_{jj}\overline{\lambda} Q_{ii}  & 0 
    \end{bmatrix} \stackrel{!}{=} \begin{bmatrix}
        0 & \lambda \\
        \overline{\lambda} & 0
   \end{bmatrix}.
\end{equation}
This reduces to
\begin{equation}\label{eq:QfixesSym_plschangedescription}
    Q_i \lambda Q_j^* = \lambda, \quad \lambda \in \F_\mu.
\end{equation}
In particular, when choosing $\lambda = 1$, we get that $Q_iQ_j^* = I$ and thus $Q_i = Q_j = U$ for some $U\in \O(\calH_\mu)$.
We need to act therefore diagonally on $\calH^\mu$ with the same orthogonal map (up to isomorphism between the copies).
Then the above Equation \ref{eq:QfixesSym_plschangedescription} reduces to
\begin{equation}
    U \lambda = \lambda U,
\end{equation}
for all $\lambda\in \F_\mu$, i.e. $U$ commutes with $\F_\mu$.
Then Lemma \ref{lem:Dlinmaps_exactly_orthunisym} tells us that those elements are exactly isomorphic to $\O(\calH_\mu)$, $\U(\C \calH_\mu)$ or $\Sp(\Hh \calH_\mu)$ depending on whether $D_\mu$ is of real, complex or quaternionic type.

Let us collect all these thoughts in the following theorem:
\begin{theorem}\label{thm:char_H1}
    We have
    \begin{equation}
        \mathsf{H}^1 \simeq \bigtimes_{n_\mu =1 } \O(\calH_\mu) \times \bigtimes_{n_\mu>1} \left(\bigtimes_{D_\mu \simeq \R} \O(\calH_\mu) \times \bigtimes_{D_\mu\simeq \C} \U\left(\C \calH_\mu\right) \times \bigtimes_{D_\mu\simeq\Hh}\Sp\left(\Hh \calH_\mu\right) \right).
    \end{equation}
    Note that each $\O(\calH_\mu)$, $\U(\calH_\mu)$ and $\Sp(\calH_\mu)$ acts diagonally on $\calH^\mu$, i.e. each element acts on each irreducible block $\calH_\mu$ simultaneously. 
\end{theorem}

\begin{example}
    Consider the compact group of unit scalars in $\C$, i.e.
    \begin{equation}
        G = S^1 = \{\exp(i\theta) : \theta\in\R \}\subseteq \C
    \end{equation}
    acting on $\calH = \R^2$ as through left multipliation on $\C\sim \R^2$. 
    That is,
    \begin{equation} \label{eq:S1_U1_SO2_iso}
        \rho(\exp(i\theta))v  =\begin{pmatrix}
            \Re \exp(i\theta) & - \Im\exp(i\theta) \\
            \Im \exp(i\theta) & \Re \exp(i\theta)
        \end{pmatrix} = \begin{pmatrix}
            \cos(\theta) & -\sin(\theta) \\
            \sin(\theta) & \cos(\theta)
        \end{pmatrix}.
    \end{equation}
    $\calH$ is a real irreducible representation of $G$.
 $\End_G(\calH) = \R I + \R J\simeq \C$, where $J$ is given by
    \begin{equation}
        J = 
        \begin{pmatrix}
            0 & -1 \\
            1 & 0
        \end{pmatrix},
    \end{equation}
    is the group of scaled rotation matrices. 
    Thus, $\calH$ is of complex type.    We note that $J^* = J^\top = -J$.

    Now define $V = \calH\oplus \calH$ and let $G$ act diagonally on $V$, i.e.
    \begin{equation*}
        g\cdot (x, y) := (g\cdot x, g\cdot y).
    \end{equation*}
    Then $C$ has two irreducible real representations which are isomorphic to each other and to $\calH$. Proposition \ref{prop:SymAreHerm} shows that $\Sym^2_G(V)$ consists of all matrices $T$ of the form
    \begin{equation}
        T = \begin{pmatrix}
            aI & bI + cJ \\
            bI - cJ & dI
        \end{pmatrix}
    \end{equation}
    with $a,b,c,d \in \R$.

    Theorem \ref{thm:char_H1} yields that the subgroup of $\O(\Tilde{V})$ fixing $\Sym^2_G(\Tilde{V})$ by conjugation contains exactly the matrices $Q$ of the form
    \begin{equation}
        Q = \begin{pmatrix}
            R & 0 \\
            0 & R
        \end{pmatrix} 
    \end{equation}
    where $R \in SO(2)$ is a rotation matrix.
\end{example}

\subsubsection{Characterizing $\mathsf{H}^2$}
Now we only need to find the maximal subgroup which fixes $\Hom_G(\calH, \calY)$ by right multiplication.
Let 
\begin{equation}
    \calY \simeq \bigoplus_\mu \calY^\mu \simeq \bigoplus_\mu \calY_\mu \otimes_\R \R^{m_\mu}
\end{equation}
be the irreducible decomposition of $\calY$ with regard to $G$ as in Equation \ref{eq:W_decomp}.
In this notation, we encode isomorphy of representations: $\calY_\mu \simeq \calH_\eta$ if and only if $\mu = \eta$.

Similarly to Proposition \ref{prop:End_decomp} we have
\begin{lemma} \label{lem:HomAreMat}
    Let $G$ be a compact group action acting on $\calH$ and $\calY$.
    Then
    \begin{equation}
        \Hom_G(\calH, \calY) \simeq \bigtimes_\mu M_{m_\mu, n_\mu}(\F_\mu)
    \end{equation}
    where we take the product over all possible isomorphic irreducible representations of $G$ indexed by $\mu$. 
\end{lemma}
\begin{proof}
    Similarly to the proof of Proposition \ref{prop:End_decomp}, we have
    \begin{align}
        \Hom_G(\calH, \calY) &\simeq \bigtimes_{\mu, \eta} \Hom_G(\calH^\mu, \calY^\eta) \simeq \bigtimes_{\mu, \eta} \Hom_G(\calH_\mu, \calY_\eta) \otimes_\R \Hom(\R^{n_\mu}, \R^{m_\eta}). 
    \end{align}
    Now, due to Schur's Lemma, $\Hom_G(\calH_\mu, \calY_\eta) \simeq \End_G(\calH_\mu)$ if $\calH_\mu \simeq \calY_\eta$ and $0$ otherwise.
    By definition of the decompositions, $\calH_\mu \simeq \calY_\eta$ precisely whenever $\mu = \eta$. 
    Therefore,
    \begin{align}
        \Hom_G(\calH, \calY) &\simeq \bigtimes_\mu \Hom(\calH^\mu, \calY^\eta) \simeq \bigtimes_\mu \End(\calH_\mu) \otimes_\R \Hom(\R^{n_\mu}, \R^{m_\mu}) \simeq \bigtimes_\mu M_{m_\mu, n_\mu}(\F_\mu)
    \end{align}
    where we use the identification $\End_G(\calH_\mu) \simeq \F_\mu$ as in the earlier sections.
\end{proof}

With this decomposition, whenever $n_\lambda, m_\lambda \neq 0$, the respective component in $\Hom_G(\calH, \calY)$ is not zero.
Let's restrict our attention to such a non-trivial block $\Hom_G(\calH^\mu, \calY^\mu)$.
Let $T = (t_{ij})_{ij} \in M_{m_\mu, n_\mu}(\F_\mu)$ and $Q \in \O(\calH^\mu)$ be given by
\begin{equation}
    Q = \begin{pmatrix}
    Q_{11} & \ldots & Q_{n_\mu, 1} \\
    \vdots & \ddots & \vdots \\
    Q_{1, n_\mu} & \ldots & Q_{n_\mu, n_\mu}
    \end{pmatrix}
\end{equation}
where each $Q_{k\ell}\in \End(\calH_\mu)$ maps on copy $\ell$ of $\calH_\mu$ onto the $k$:th copy.
Then, $T = TQ$ if and only if
\begin{equation}
    t_{ij} = \sum_k t_{ik} Q_{kj}
\end{equation}
for all $(i,j)$.
This is equivalent to
\begin{equation}
    t_{ij}(I - Q_{jj}) = \sum_{k\neq j}t_{ik} Q_{kj}. \label{eq:qeq}
\end{equation}
Due to Lemma \ref{lem:HomAreMat}, the $t_{ij}\in \F$ can be chosen freely. Consequently, \eqref{eq:qeq} yields that $Q_{kk} = I$ and $Q_{kl} = 0$ whenever $k\neq \ell$.
To conclude, $Q = I_{\calH^\mu}$ which obviously fixes all elements in $\Hom_G(\calH^\mu, \calY^\mu)$ by right multiplication.

As a result, we have
\begin{prop}\label{prop:char_H2}
    Let $I = \{\mu: n_\mu , m_\mu > 0\}$ be the sets of indices where $\Hom(\calH^\mu, \calY^\mu)$ doesn't vanish and $I' = \{\mu: n_\mu > 0, m_\mu = 0 \}$.
    Set $\calH^{\mathrm{fixed}} = \bigoplus_{\mu\in I} \calH^\mu$ and $\calH^{\mathrm{free}} = \bigoplus_{\mu\in I'} \calH^\mu$
    \begin{equation}
        \mathsf{H}^2 \simeq \{I_{\calH^{\mathrm{fixed}}}\} \times \O(\calH^{\mathrm{free}})
    \end{equation}
\end{prop}
\begin{proof}
    We have an orthogonal decomposition $\calH =\calH^{\mathrm{fixed}} \oplus \calH^{\mathrm{free}}$.
    In particular, we have
    \begin{equation}
        \Hom_G(\calH, \calY) \simeq \Hom_G(\calH^{\mathrm{fixed}}, \calY) \times \{0_{\calH^{\mathrm{free}}, \calY}\}
    \end{equation}
    So, our elements in $\mathsf{H}^2$ only need to fix elements in $\Hom_G(\calH^{\mathrm{fixed}})$ and can act freely on $\calH^{\mathrm{free}}$.
    
    We have shown above that the elements $Q$ fixing $\Hom_G(\calH^{\mathrm{fixed}})$ by right multiplication are those that fix $\calH^{\mathrm{fixed}}$, i.e. $Q|_{\calH^{\mathrm{fixed}}} = I_{\calH^{\mathrm{fixed}}}$.
    Because $Q$ is orthogonal, this implies that $Q(\calH^{\mathrm{fixed}}) = \calH^{\mathrm{fixed}}$ and $Q(\calH^{\mathrm{free}}) = \calH^{\mathrm{free}}$.
    In particular, $Q = I_{\calH^{\mathrm{fixed}}} \times Q'$ for some $Q' \in \O(\calH^{\mathrm{free}})$ and every element of this form obviously fixes $\Hom_G(\calH, \calY)$.
\end{proof}

\subsubsection{Characterizing $\mathsf{H}$}
Now combining Theorem \ref{thm:char_H1} and Proposition \ref{prop:char_H2}, we get
\begin{theorem}
    Let $I = \{\mu: n_\mu , m_\mu > 0\}$ be the sets of indices where $\Hom(\calH^\mu, \calY^\mu)$ doesn't vanish and $I' = \{\mu: n_\mu > 0, m_\mu = 0 \}$.
    Set $\calH^{\mathrm{fixed}} = \bigoplus_{\mu\in I} \calH^\mu$.
    Then
    \begin{equation}
        \mathsf{H} \simeq \{I_{\calH^{\mathrm{fixed}}}\} \times \bigtimes_{\mu\in I' } \left(\bigtimes_{n_\mu =1 } \O(\calH_\mu) \times \bigtimes_{n_\mu>1} \left(\bigtimes_{D_\mu \simeq \R} \O(\calH_\mu) \times \bigtimes_{D_\mu\simeq \C} U\left(\C \calH_\mu\right) \times \bigtimes_{D_\mu\simeq\Hh}Sp\left(\Hh \calH_\mu\right)\right) \right),
    \end{equation}
    where we note that each $\O(\calH_\mu)$, $\U(\C\calH_\mu)$ and $\Sp(\Hh\calH_\mu)$ acts diagonally on $\calH^\mu$, i.e. each element acts on each irreducible block $\calH_\mu$ simultaneously. 
\end{theorem}

\newpage

\section{\texorpdfstring{Calculation of $\mathsf{H}$ for some concrete actions}{Calculation of H for some concrete actions}}
 Let us calculate the $\mathsf{H}$-group in two simple cases -- the action of $C_3$ on $\R^3$ from Example \ref{ex:linearmodel}, and the $C_4$-group rotating pixel images in $2D$.

\paragraph{$C_3$ acting on $\R^3$}\label{sec:C3actingR3}
We are following Example \ref{ex:linearmodel} giving a quick run down of the classification and following it up with a more detailed discussion.

Let $C_3 = \{e, r, r^2\}$ be the cyclic group acting on $\R^3$ through permutation of the indices, i.e. $re_i := e_{(i+1)\, \mathrm{mod} \, 3}$.
This action is obviously orthogonal with respect to the standard inner product on $\R^3$. 
We obtain the following orthogonal irreducible components
\begin{equation}
    \R^3 = \R \one \oplus (\R \one )^\perp,
\end{equation}
with $\R \one $ being of real type and $(\R\one )^\perp$ being of complex type.
As a result, $\mathsf{H}^1$ is given by
\begin{equation}
    \mathsf{H}^1 \simeq \O(\R \one ) \times \U\left( \C (\R \one )^\perp \right) \simeq \{\pm I_{\R\one }\} \times \SO\left( (\R \one )^\perp \right).
\end{equation}

Furthermore, since we act trivially on $\calY$ and also on $\R\one $, we have
\begin{equation}
    \mathsf{H}^2 = \{ I_{\R\one } \}\times \SO\left(\left(\R\one \right)^\perp\right). 
\end{equation}
Therefore, we have
\begin{equation}
    \mathsf{H} = \mathsf{H}^1\cap \mathsf{H}^2 = \{I_{\R\one }\} \times \SO\left(\left(\R\one \right)^\perp\right) \simeq \SO\left(\left(\R\one \right)^\perp\right).
\end{equation}
Now, let us add the details to the claims above.
Let us start with how we classified the real decomposition of $\R^3$.
As $C_3$ is generated by the element $r$, its complex irreducible components are given by the eigenspaces of $r$.
The characteristic polynomial of $r$ is $X^3 - 1$, so we have three distinct eigenvalues 
\begin{equation}
    1, \xi_3= \exp(2\pi i/3), \  \xi_3^2 = \exp(4\pi i/3),
\end{equation}
the cube roots of unity.
The real eigenspace of $r$ to $1$ is $\R\one $, while the complex ones lie in $\left((\R \one )^\perp\right)_\C$ (
the complex span of the vectors in $(\R\one)^\perp$).
The primitive roots $\xi_3, \xi_3^2$ are complex conjugates of each other, and so are their (complex) eigenspaces complex conjugates too.

Let $w$ be an eigenvector of $\xi_3$ in $\left((\R \one )^\perp\right)_\C$.
Then $\overline{w}$ is an eigenvector to $\xi_3^2$ and we get a real basis of $\left((\R \one )^\perp\right)_\C$ as 
\begin{equation}
    \Re w = w + \overline{w} , \  \Im w = \ j\overline{w} - jw.
\end{equation}
Define $J: \left((\R \one )^\perp\right)_\C \to \left((\R \one )^\perp\right)_\C$ by acting as multiplication of $j$ on the eigenspace of $\xi_3$ and $-j$ on the eigenspace of $\xi_3^2$.
Then
\begin{align}
    \overline{J(w + \overline{w})} = \overline{jw - j\overline{w}} = jw - j\overline{w} = J(w + \overline{w}) \\
    \overline{J(j\overline{w} - jw)} = \overline{w + \overline{w}} = w + \overline{w} = J(j\overline{w} - jw)
\end{align}
so $J$ is real valued on $(\R \one )^\perp$.
In our case, the real and imaginary parts of the normalized eigenvectors, and thus a basis of $(\R \one )^\perp$ are given by
\begin{align}
    v_\Re = \frac{\sqrt{6}}{6}\begin{bmatrix}
    2 & -1 & -1
\end{bmatrix}^T, \ v_\Im = \frac{\sqrt{2}}{2}\begin{bmatrix}
    0 & - 1 & 1
\end{bmatrix}^T
\end{align}
Therefore, $J$ has the following form
\begin{equation}
    J = \begin{bmatrix}
         v_\Re & v_\Im
    \end{bmatrix} \begin{bmatrix}
         0 & -1 \\
         1 & 0
    \end{bmatrix} \begin{bmatrix}
        v_\Re & v_\Im
    \end{bmatrix}^\top = \frac{\sqrt{3}}{3} \begin{bmatrix}
        0 & 1 & -1 \\
        -1 & 0 & 1 \\
        1 & -1 & 0
    \end{bmatrix}
\end{equation}
which is circulant and therefore commutes with $C_3$.
On $(\R \one )^\perp$ it has the standard form
\begin{equation}
    J = \begin{bmatrix}
        0 & -1 \\
        1 & 0
    \end{bmatrix}
\end{equation}
with respect to the basis $v_\Re$ and $v_\Im$.

Using the same basis, we have the following
\begin{align}
    \mathsf{H} \simeq \left\{ \begin{bmatrix}
        \frac{\sqrt{3}}{3}\one & v_\Re & v_\Im
    \end{bmatrix} \begin{bmatrix}
          1 & 0 & 0 \\
         0 & a & \mp b \\
         0 & b &  \pm a
    \end{bmatrix} \begin{bmatrix}
        \frac{\sqrt{3}}{3}\one & v_\Re & v_\Im
    \end{bmatrix}^\top \, \Big\vert \, a^2 + b^2 = 1\right\}.
\end{align}

\paragraph{Rotating an image grid}
Consider $C_4 = \{e, r, r^2, r^3\}$ acting on an image grid of size $2D\times 2D$. The rotation action naturally decomposes the grid into $D$ orbits of size 4
(cf. Figure \ref{fig:rotationonagrid}). Hence, it is fruitful to think about the space of images as $\R^4\otensor \R^{D}$. $C_4$ then acts via cyclic permutation onto the $\R^4$ component, i.e.
\begin{equation*}
    r (e_i \otensor x) := e_{(i+1)\, \mathrm{mod} \, 4} \otensor x
\end{equation*}
for all $x\in \R^D$. In particular, each component $\R^4 \otensor_\R e_i$ for $i = 1, \ldots, D$ is fixed by the $C_4$-action.
The components further decompose into irreps as follows 
\begin{equation} \label{eq:orbit_decomp}
    \R^4  = \R\one \oplus \R \one _{\mathrm{alt}} \oplus \mathrm{span}(\psi_1,\psi_2),
\end{equation}
where $\one_{\mathrm{alt}}= [1,-1,1,-1]^\top$, and $\psi_1=[1,0,-1,0]^\top$ and $\psi_2= [0,1,0,-1]^\top$, see also Figure \ref{fig:gridactionrepresentation}. The spaces $\R\one$ and $\R\one_{\mathrm{alt}}$ are real representations, whereas $V:=\mathrm{span}(\psi_1,\psi_2)$ is complex. 

One proves this similarly to the $C_3$ example: The characteristic polynomial of $r$ is now
\begin{equation}
    X^4 - 1 = (X+1)(X-1)(X^2 + 1),
\end{equation}
so that there are two real eigenspaces/irreps $\one$ and $\one_{\mathrm{alt}}$, and two complex ones, yielding the complex irrep $V:= \mathrm{span}(\psi_1,\psi_2)$. 

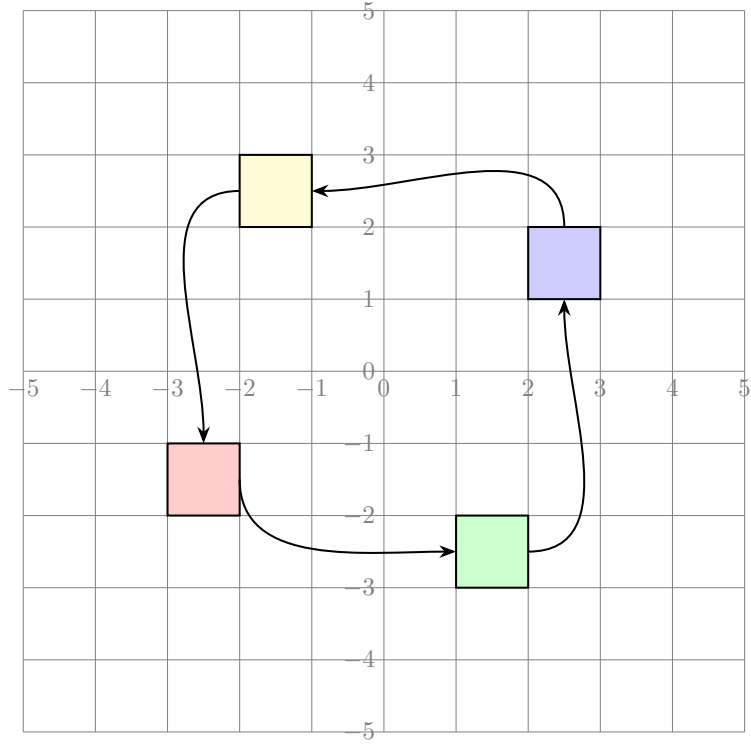
\begin{figure}
    \centering
    \resizebox{0.6\textwidth}{!}{
    \begin{tikzpicture}
        \newcommand\xMin{-5}
        \newcommand\xMax{5}
        \newcommand\yMin{-5}
        \newcommand\yMax{5}
        \newcommand\xRec{2}
        \newcommand\yRec{1}
        
        \foreach \i in {\xMin,...,\xMax} {
            \draw [very thin,gray] (\i,\yMin) -- (\i,\yMax)  node [below] at (\i, 0) {$\i$};
        }
        \foreach \i in {\yMin,...,\yMax} {
            \draw [very thin,gray] (\xMin,\i) -- (\xMax,\i) node [left] at (0,\i) {$\i$};
        }

        \fill[blue!20] (\xRec, \yRec) rectangle (\xRec+1, \yRec+1);
        \fill[yellow!20] (-\yRec, \xRec) rectangle (-\yRec-1, \xRec+1);
        \fill[red!20] (-\xRec, -\yRec) rectangle (-\xRec-1, -\yRec-1);
        \fill[green!20] (\yRec, -\xRec ) rectangle (\yRec+1, -\xRec -1);
    
        \draw[thick] (\xRec, \yRec) rectangle (\xRec+1, \yRec+1);
        \draw[thick] (-\yRec, \xRec) rectangle (-\yRec-1, \xRec+1);
        \draw[thick] (-\xRec, -\yRec) rectangle (-\xRec-1, -\yRec-1);
        \draw[thick] (\yRec, -\xRec ) rectangle (\yRec+1, -\xRec -1);
    
        \draw[->, thick, >=Stealth] (\xRec + 0.5,\yRec + 1) 
            to[out=90,in=0] (-\yRec, \xRec + 0.5);
    
        \draw[->, thick, >=Stealth] (-\yRec - 1, \xRec + 0.5) 
            to[out=180,in=90] (-\xRec - 0.5, -\yRec);
    
        \draw[->, thick, >=Stealth] (-\xRec, -\yRec - 0.5) 
            to[out=270,in=180] (\yRec, -\xRec - 0.5);
    
        \draw[->, thick, >=Stealth] (\yRec + 1, -\xRec - 0.5) 
            to[out=0,in=270] (\xRec + 0.5, \yRec);
    \end{tikzpicture}
    }
    \caption{Rotation orbit of a grid cell.}
    \label{fig:rotationonagrid}
\end{figure}

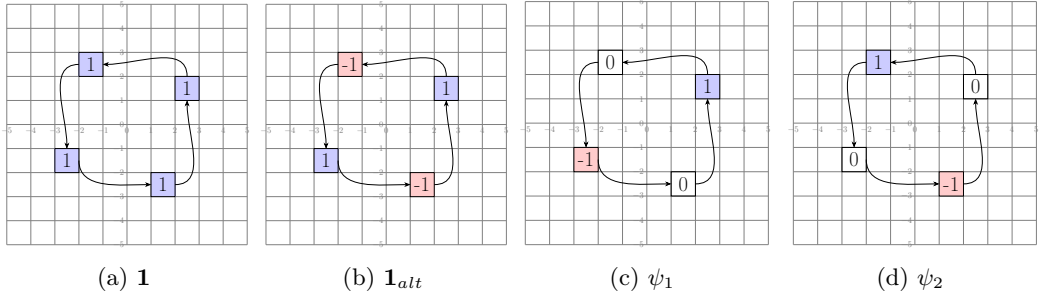
\begin{figure}[ht]
    \centering
    \begin{subfigure}{0.2\textwidth}
    \resizebox{\linewidth}{!}{
        \begin{tikzpicture}[scale=1]
            \newcommand\xMin{-5}
            \newcommand\xMax{5}
            \newcommand\yMin{-5}
            \newcommand\yMax{5}
            \newcommand\xRec{2}
            \newcommand\yRec{1}
            
            \foreach \i in {\xMin,...,\xMax} {
                \draw [very thin,gray] (\i,\yMin) -- (\i,\yMax)  node [below] at (\i, 0) {$\i$};
            }
            \foreach \i in {\yMin,...,\yMax} {
                \draw [very thin,gray] (\xMin,\i) -- (\xMax,\i) node [left] at (0,\i) {$\i$};
            }

            \fill[blue!20] (\xRec, \yRec) rectangle (\xRec+1, \yRec+1);
            \fill[blue!20] (-\yRec, \xRec) rectangle (-\yRec-1, \xRec+1);
            \fill[blue!20] (-\xRec, -\yRec) rectangle (-\xRec-1, -\yRec-1);
            \fill[blue!20] (\yRec, -\xRec ) rectangle (\yRec+1, -\xRec -1);
        
            \draw[thick] (\xRec, \yRec) rectangle (\xRec+1, \yRec+1);
            \draw[thick] (-\yRec, \xRec) rectangle (-\yRec-1, \xRec+1);
            \draw[thick] (-\xRec, -\yRec) rectangle (-\xRec-1, -\yRec-1);
            \draw[thick] (\yRec, -\xRec ) rectangle (\yRec+1, -\xRec -1);
            
            \node[font=\Huge,align=center] at (\xRec+0.5, \yRec+0.5) {1};
            \node[font=\Huge,align=center] at (-\yRec-0.5, \xRec+0.5) {1};
            \node[font=\Huge,align=center] at (-\xRec-0.5, -\yRec-0.5) {1};
            \node[font=\Huge,align=center] at (\yRec+0.5, -\xRec-0.5) {1};
            
            \draw[->, thick, >=Stealth] (\xRec + 0.5,\yRec + 1) 
                to[out=90,in=0] (-\yRec, \xRec + 0.5);
            \draw[->, thick, >=Stealth] (-\yRec - 1, \xRec + 0.5) 
                to[out=180,in=90] (-\xRec - 0.5, -\yRec);
            \draw[->, thick, >=Stealth] (-\xRec, -\yRec - 0.5) 
                to[out=270,in=180] (\yRec, -\xRec - 0.5);
            \draw[->, thick, >=Stealth] (\yRec + 1, -\xRec - 0.5) 
                to[out=0,in=270] (\xRec + 0.5, \yRec);
        \end{tikzpicture}
        }
        \caption{$\one $}
        \label{fig:real1}
    \end{subfigure}%
    \begin{subfigure}{0.2\textwidth}
    \resizebox{\linewidth}{!}{
        \begin{tikzpicture}[scale=1]
            \newcommand\xMin{-5}
            \newcommand\xMax{5}
            \newcommand\yMin{-5}
            \newcommand\yMax{5}
            \newcommand\xRec{2}
            \newcommand\yRec{1}
            
            \foreach \i in {\xMin,...,\xMax} {
                \draw [very thin,gray] (\i,\yMin) -- (\i,\yMax)  node [below] at (\i, 0) {$\i$};
            }
            \foreach \i in {\yMin,...,\yMax} {
                \draw [very thin,gray] (\xMin,\i) -- (\xMax,\i) node [left] at (0,\i) {$\i$};
            }

            \fill[blue!20] (\xRec, \yRec) rectangle (\xRec+1, \yRec+1);
            \fill[red!20] (-\yRec, \xRec) rectangle (-\yRec-1, \xRec+1);
            \fill[blue!20] (-\xRec, -\yRec) rectangle (-\xRec-1, -\yRec-1);
            \fill[red!20] (\yRec, -\xRec ) rectangle (\yRec+1, -\xRec -1);
        
            \draw[thick] (\xRec, \yRec) rectangle (\xRec+1, \yRec+1);
            \draw[thick] (-\yRec, \xRec) rectangle (-\yRec-1, \xRec+1);
            \draw[thick] (-\xRec, -\yRec) rectangle (-\xRec-1, -\yRec-1);
            \draw[thick] (\yRec, -\xRec ) rectangle (\yRec+1, -\xRec -1);
        
            \node[font=\Huge,align=center] at (\xRec+0.5, \yRec+0.5) {1};
            \node[font=\Huge,align=center] at (-\yRec-0.5, \xRec+0.5) {-1};
            \node[font=\Huge,align=center] at (-\xRec-0.5, -\yRec-0.5) {1};
            \node[font=\Huge,align=center] at (\yRec+0.5, -\xRec-0.5) {-1};
        
            \draw[->, thick, >=Stealth] (\xRec + 0.5,\yRec + 1) 
                to[out=90,in=0] (-\yRec, \xRec + 0.5);
            \draw[->, thick, >=Stealth] (-\yRec - 1, \xRec + 0.5) 
                to[out=180,in=90] (-\xRec - 0.5, -\yRec);
            \draw[->, thick, >=Stealth] (-\xRec, -\yRec - 0.5) 
                to[out=270,in=180] (\yRec, -\xRec - 0.5);
            \draw[->, thick, >=Stealth] (\yRec + 1, -\xRec - 0.5) 
                to[out=0,in=270] (\xRec + 0.5, \yRec);
        
        \end{tikzpicture}
        }
        \caption{$\one _{alt}$}
    \label{fig:real-1}
    \end{subfigure}%
    \begin{subfigure}{0.2\textwidth}
        \resizebox{\linewidth}{!}{
        \begin{tikzpicture}[scale=1]
            \newcommand\xMin{-5}
            \newcommand\xMax{5}
            \newcommand\yMin{-5}
            \newcommand\yMax{5}
            \newcommand\xRec{2}
            \newcommand\yRec{1}
            
            \foreach \i in {\xMin,...,\xMax} {
                \draw [very thin,gray] (\i,\yMin) -- (\i,\yMax)  node [below] at (\i, 0) {$\i$};
            }
            \foreach \i in {\yMin,...,\yMax} {
                \draw [very thin,gray] (\xMin,\i) -- (\xMax,\i) node [left] at (0,\i) {$\i$};
            }

            \fill[blue!20] (\xRec, \yRec) rectangle (\xRec+1, \yRec+1);
            \fill[red!20] (-\xRec, -\yRec) rectangle (-\xRec-1, -\yRec-1);
        
            \draw[thick] (\xRec, \yRec) rectangle (\xRec+1, \yRec+1);
            \draw[thick] (-\yRec, \xRec) rectangle (-\yRec-1, \xRec+1);
            \draw[thick] (-\xRec, -\yRec) rectangle (-\xRec-1, -\yRec-1);
            \draw[thick] (\yRec, -\xRec ) rectangle (\yRec+1, -\xRec -1);
        
            \node[font=\Huge,align=center] at (\xRec+0.5, \yRec+0.5) {1};
            \node[font=\Huge,align=center] at (-\yRec-0.5, \xRec+0.5) {0};
            \node[font=\Huge,align=center] at (-\xRec-0.5, -\yRec-0.5) {-1};
            \node[font=\Huge,align=center] at (\yRec+0.5, -\xRec-0.5) {0};
        
            \draw[->, thick, >=Stealth] (\xRec + 0.5,\yRec + 1) 
                to[out=90,in=0] (-\yRec, \xRec + 0.5);
            \draw[->, thick, >=Stealth] (-\yRec - 1, \xRec + 0.5) 
                to[out=180,in=90] (-\xRec - 0.5, -\yRec);
            \draw[->, thick, >=Stealth] (-\xRec, -\yRec - 0.5) 
                to[out=270,in=180] (\yRec, -\xRec - 0.5);
            \draw[->, thick, >=Stealth] (\yRec + 1, -\xRec - 0.5) 
                to[out=0,in=270] (\xRec + 0.5, \yRec);
        
        \end{tikzpicture}}
        \caption{$\psi_1$}
        \end{subfigure}
         \begin{subfigure}{0.2\textwidth}
        \resizebox{\linewidth}{!}{
        \begin{tikzpicture}[scale=1]
            \newcommand\xMin{-5}
            \newcommand\xMax{5}
            \newcommand\yMin{-5}
            \newcommand\yMax{5}
            \newcommand\xRec{2}
            \newcommand\yRec{1}
            
            \foreach \i in {\xMin,...,\xMax} {
                \draw [very thin,gray] (\i,\yMin) -- (\i,\yMax)  node [below] at (\i, 0) {$\i$};
            }
            \foreach \i in {\yMin,...,\yMax} {
                \draw [very thin,gray] (\xMin,\i) -- (\xMax,\i) node [left] at (0,\i) {$\i$};
            }

            \fill[blue!20] (-\yRec, \xRec) rectangle (-\yRec-1, \xRec+1);
            \fill[red!20] (\yRec, -\xRec ) rectangle (\yRec+1, -\xRec -1);
        
            \draw[thick] (\xRec, \yRec) rectangle (\xRec+1, \yRec+1);
            \draw[thick] (-\yRec, \xRec) rectangle (-\yRec-1, \xRec+1);
            \draw[thick] (-\xRec, -\yRec) rectangle (-\xRec-1, -\yRec-1);
            \draw[thick] (\yRec, -\xRec ) rectangle (\yRec+1, -\xRec -1);
        
            \node[font=\Huge,align=center] at (\xRec+0.5, \yRec+0.5) {0};
            \node[font=\Huge,align=center] at (-\yRec-0.5, \xRec+0.5) {1};
            \node[font=\Huge,align=center] at (-\xRec-0.5, -\yRec-0.5) {0};
            \node[font=\Huge,align=center] at (\yRec+0.5, -\xRec-0.5) {-1};
        
            \draw[->, thick, >=Stealth] (\xRec + 0.5,\yRec + 1) 
                to[out=90,in=0] (-\yRec, \xRec + 0.5);
            \draw[->, thick, >=Stealth] (-\yRec - 1, \xRec + 0.5) 
                to[out=180,in=90] (-\xRec - 0.5, -\yRec);
            \draw[->, thick, >=Stealth] (-\xRec, -\yRec - 0.5) 
                to[out=270,in=180] (\yRec, -\xRec - 0.5);
            \draw[->, thick, >=Stealth] (\yRec + 1, -\xRec - 0.5) 
                to[out=0,in=270] (\xRec + 0.5, \yRec);
        
        \end{tikzpicture}}
        \caption{$\psi_2$}
        \end{subfigure}
    \label{fig:complex} 
    \caption{The four vectors in Equation \ref{eq:orbit_decomp}. Note that rotating $\one$ fixes it, $\one_{\mathrm{alt}}$ brings it to its negative, $\psi_1$ onto $\psi_2$ and $\psi_2$ onto $-\psi_1$.}
    \label{fig:gridactionrepresentation}

\end{figure}

Now note that within each orbit, the alternating representation on $\R\one_{\mathrm{alt}}$ is not isomorphic to the trivial one on $\R\one$. We do have isomorphy across the orbits, however. Theorem \ref{thm:char_H1} therefore yields
\begin{equation}
    \mathsf{H}^1 \simeq \{\pm I_{\R\one}\}\times \{\pm I_{\R\one_{alt}}\} \times \U(\C V).
\end{equation}
Note that the action of $\U(\C V)$ can be represented conveniently in the basis $\psi_1,\psi_2$ -- the $J$-operator corresponds to the matrix
\begin{align*}
    \begin{bmatrix}
        0 & -1 \\
        1 & 0 
    \end{bmatrix}.
\end{align*}

Now, if we assume that $G$ acts trivially on $\calY$, Proposition \ref{prop:char_H2} yields
\begin{equation}
    \mathsf{H}^2 = \{I_{\R\one}\otensor I_{\R^D}\} \times \O\left(\left(V\oplus \R\one_{alt}\right)\otensor \R^D\right) .
\end{equation}
As a result,
\begin{equation}
    \mathsf{H} = \mathsf{H}^1\cap \mathsf{H}^2 = \left(\{I_{\R\one}\}\times \{\pm I_{\R\one_{alt}}\} \times \U(\C V)\right)\otensor I_{\R^D} \simeq \{\pm 1\}\times \U(\C V).
\end{equation}

\end{document}